\documentclass[journal]{IEEEtran}
\ifCLASSINFOpdf
\else
\fi
\usepackage{graphicx}

\usepackage{algpseudocode}
\usepackage[linesnumbered,ruled,vlined, noend]{algorithm2e}
\usepackage{etex}

\usepackage{amssymb}
\usepackage{mathrsfs}
\usepackage{amsmath}

\usepackage{amsthm}

\newtheorem{theorem}{Theorem}

\usepackage{multirow}
\usepackage{mathtools}
\usepackage[sans]{dsfont}
\usepackage{stmaryrd}
\usepackage{pifont}
\usepackage{wasysym}

\usepackage{array}
\usepackage{url}
\usepackage{colortbl}
\usepackage{caption}
\usepackage{subcaption}
\usepackage{color}
\usepackage{diagbox,pict2e}

\newtheorem{lemma}{Lemma}

\theoremstyle{definition}

\def\cA{\mathcal A}

\def\cJ{\mathcal J}

\def\cM{\mathcal M}
\def\cN{\mathcal N}

\def\cS{\mathcal S}

\def\cY{\mathcal Y}

\def\cX{\mathcal X}

\def\RR{\mathbb R}
\def\RR{\mathbb R}

\def\bx{\mathbf{x}}

\usepackage{cite}
\newcommand{\raf}[1]{(\ref{#1})}

\usepackage[hidelinks]{hyperref}

\begin{document}
%
\title{Multi-Agent Chance-Constrained Stochastic Shortest Path with Application to Risk-Aware Intelligent Intersection}

%

\author{Majid Khonji*, Rashid Alyassi*, Wolfgang Merkt, Areg Karapetyan, 
 Xin Huang, Sungkweon Hong, Jorge Dias, and~Brian Williams 
 \thanks{*These authors contributed equally.}
\thanks{This work was supported by the Khalifa University of Science and Technology under Award Ref. CIRA-2020-286.}
\thanks{M. Khonji, R. Alyassi, A. Karapetyan and J. Dias are with the EECS Department, Khalifa University, Abu Dhabi, UAE. (e-mails: \{majid.khonji, rashid.alyassi, areg.karapetyan, jorge.dias\}@ku.ac.ae)}
\thanks{W. Merkt is with the Oxford Robotics Institute, University of Oxford. (e-mail: wolfgang@robots.ox.ac.uk)}
\thanks{X. Huang, S. Hong and B. Williams are with CSAIL, Massachusetts Institute of Technology, Cambridge, MA, USA. (e-mails: \{huangxin, sk5050, williams\}@mit.edu)}}

\markboth{} 
{Shell \MakeLowercase{\textit{et al.}}: Multi-Agent Chance-Constrained Stochastic Shortest Path with Application to Risk-Aware Intelligent Intersection}

%



\maketitle

\IEEEpeerreviewmaketitle

\begin{abstract}
In transportation networks, where traffic lights have traditionally been used for vehicle coordination, intersections act as natural bottlenecks. A formidable challenge for existing automated intersections lies in detecting and reasoning about uncertainty from the operating environment and human-driven vehicles. In this paper, we propose a risk-aware intelligent intersection system for autonomous vehicles (AVs) as well as human-driven vehicles (HVs). We cast the problem as a novel class of Multi-agent Chance-Constrained Stochastic Shortest Path (MCC-SSP) problems and devise an exact Integer Linear Programming (ILP) formulation that is scalable in the number of agents' interaction points (e.g., potential collision points at the intersection). In particular, when the number of agents within an interaction point is small, which is often the case in intersections, the ILP has a polynomial number of variables and constraints. To further improve the running time performance, we show that the collision risk computation can be performed offline. Additionally, a trajectory optimization workflow is provided to generate risk-aware trajectories for any given intersection. The proposed framework is implemented in CARLA simulator and evaluated under a fully autonomous intersection with AVs only as well as in a hybrid setup with a signalized intersection for HVs and an intelligent scheme for AVs. As verified via simulations, the featured approach improves intersection's efficiency by up to $200\%$ while also conforming to the specified tunable risk threshold.


\end{abstract}



\section{Introduction}
Existing vehicle coordination methods, which are designed primarily for human drivers, tend to fall short in leveraging the increased sensitivity and precision of autonomous vehicles (AVs). With the progress of self-driving technologies, the bottleneck of roadway efficiency will no longer be attributed to drivers but rather to the automation scheme underpinning the coordination of AVs' actions. A crucial challenge for these schemes lies in detecting and reasoning about uncertainties in the operating environment. In urban scenarios, uncertainty arises predominantly\footnote{Other sources of uncertainty could stem from vehicle perception and trajectory tracking error due to road and weather conditions.} from human-driven vehicles (HVs) as their intention could exhibit a stochastic and oftentimes risky behavior. Unlike streets/highways with well-delimited lanes, intersections typically lack clear marking, thereby creating ``conflict zones'' with elevated potential for crashes. In fact, as reported in~\cite{grembek2018introducing}, 40\% of all crashes and 20\% of fatalities happen to occur at intersections.

\begin{figure*}[t!]
	\begin{subfigure}{.254\textwidth}
		\centering
     \includegraphics[trim={2.3cm 0 2.3cm 0},clip, width=\textwidth]{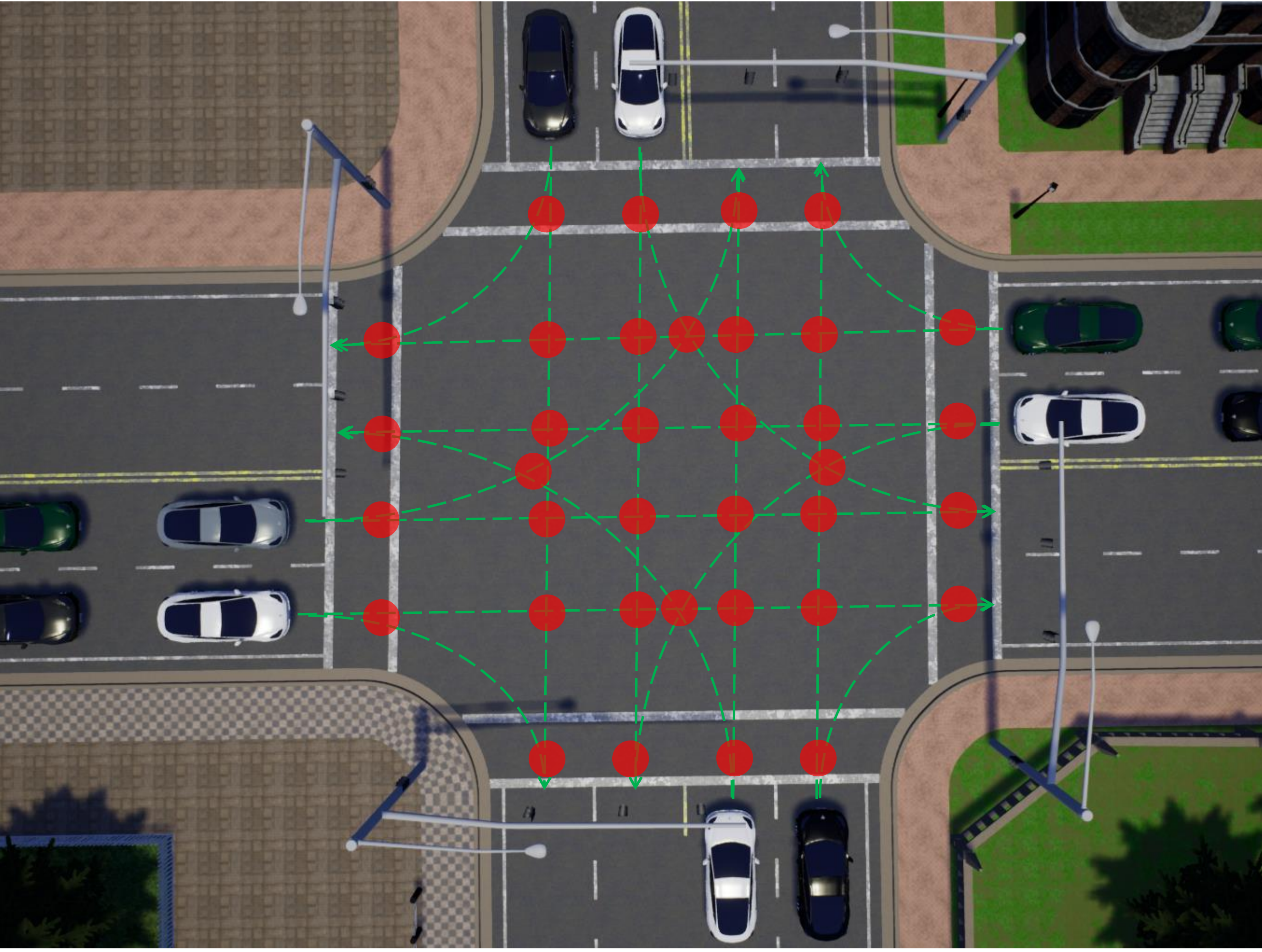}
     \caption{}
     \label{fig:colide-points}
     \end{subfigure}
	\begin{subfigure}{.24\textwidth}
		\centering
		\includegraphics[width=\textwidth]{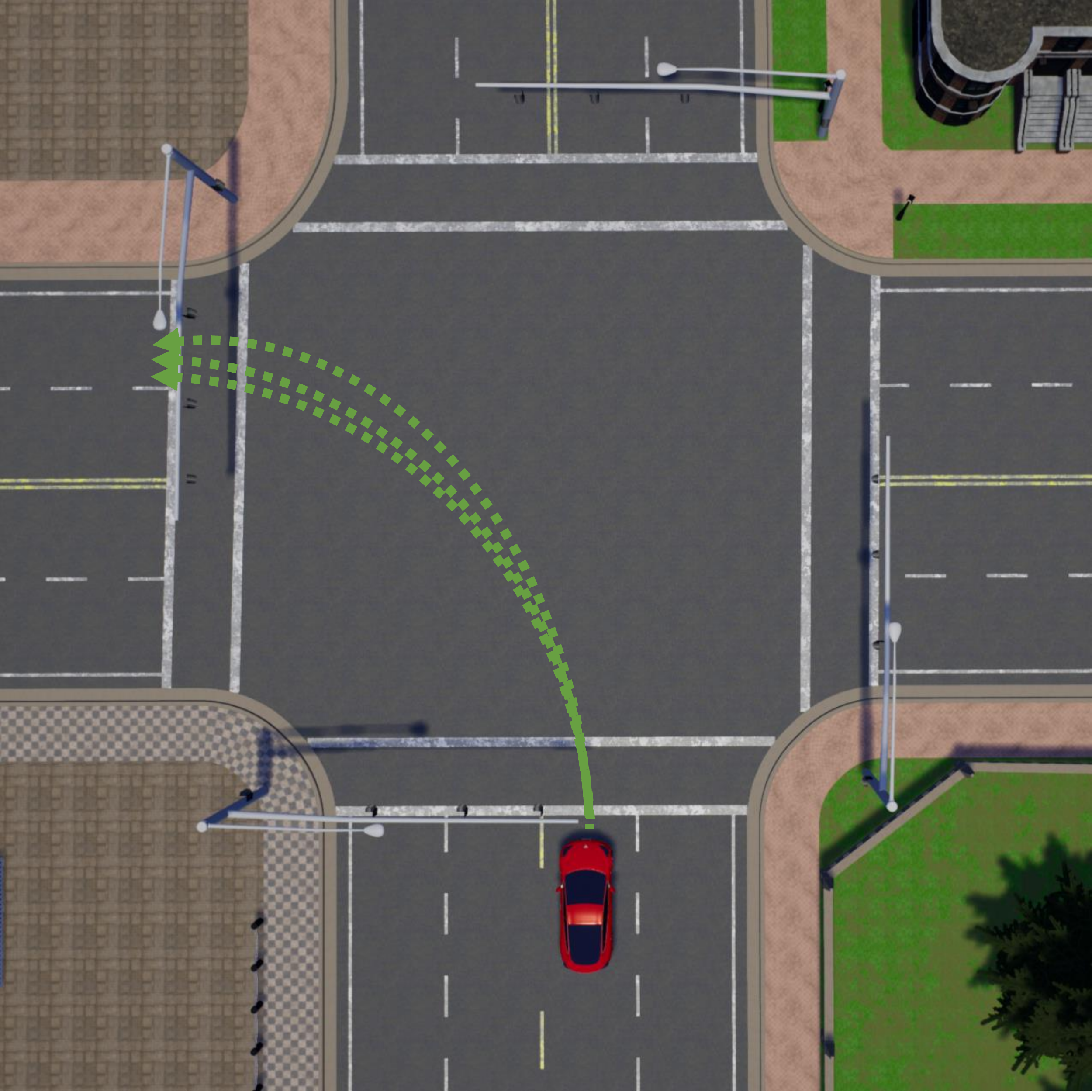}
		\caption{}\label{fig:test2}
	\end{subfigure}
	\begin{subfigure}{.24\textwidth}
		\centering
		\includegraphics[width=\textwidth]{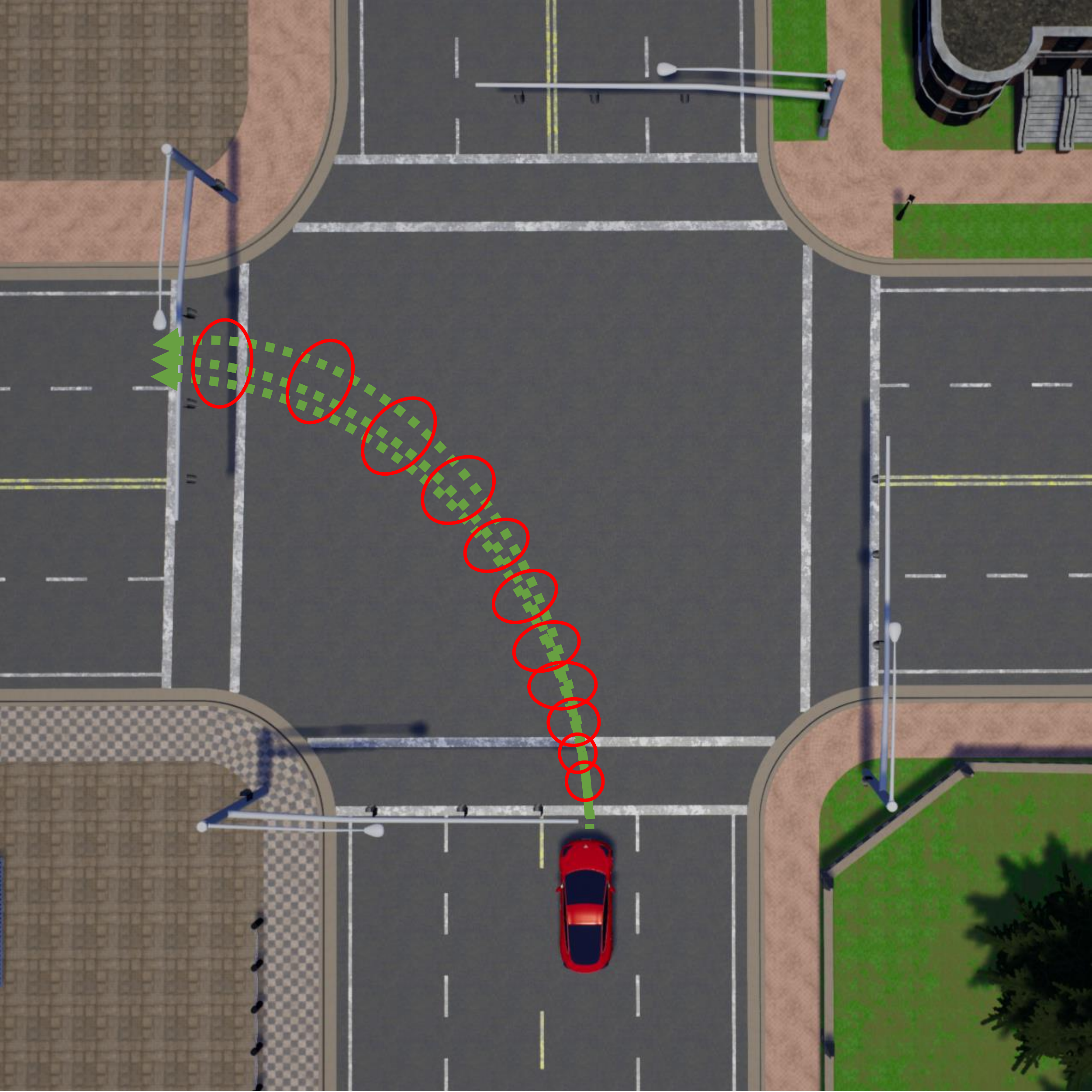}
		\caption{}\label{fig:pft}
	\end{subfigure}
	\begin{subfigure}{.24\textwidth}
		\centering
		\includegraphics[width=\textwidth]{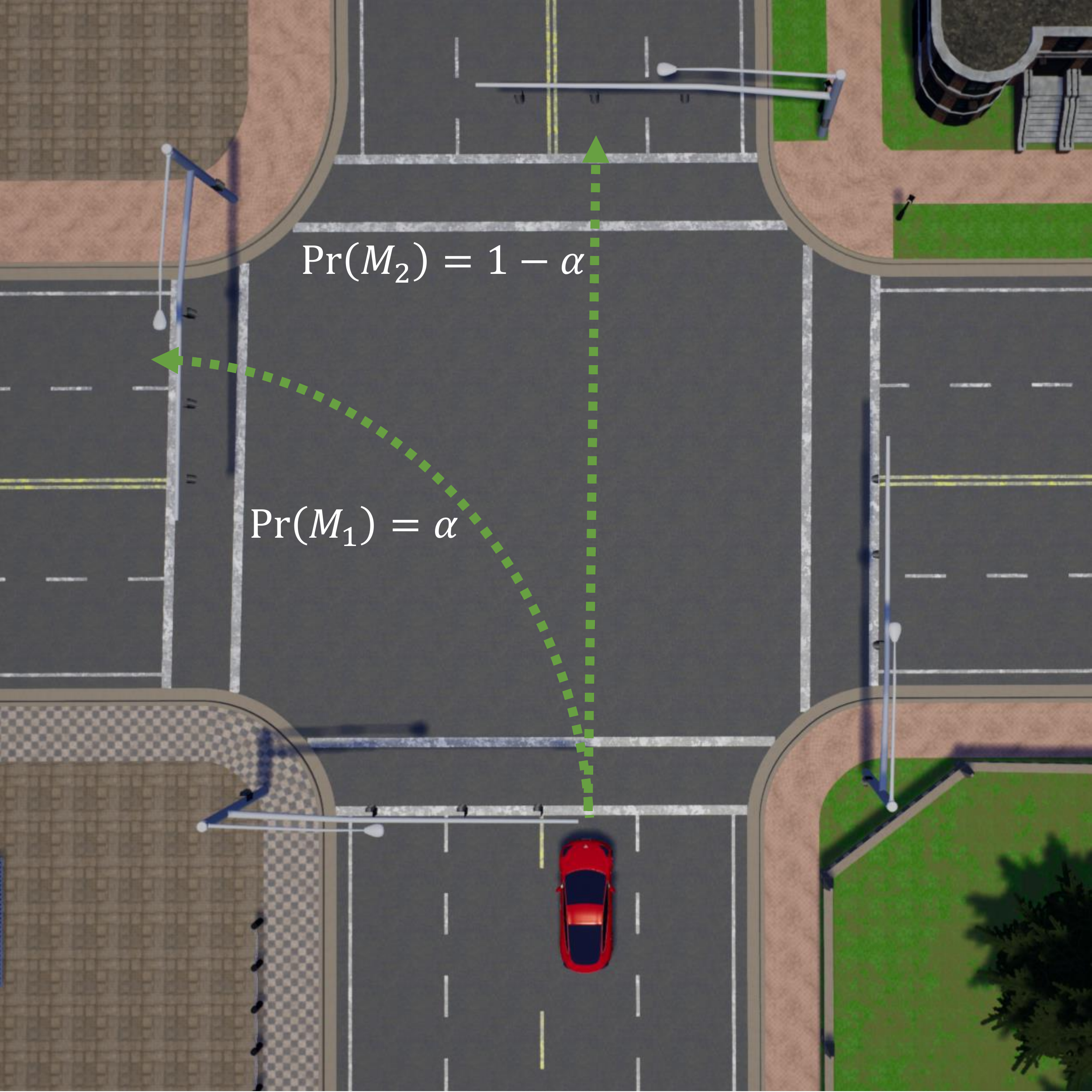}
		\caption{}\label{fig:intent}
	\end{subfigure}
	\caption{(a) The set of interaction points (collision points) between agents at an intersection. 
	(b) Multiple trajectory generation for the same maneuver. (c) PFT representation of a maneuver with tracking error. (d) Intent recognition output for an HV.}
\end{figure*}

These uncertainties can be mitigated by a \textit{risk-aware Intelligent Intersection} system in which AVs' actions are judiciously planned and coordinated in a centralized or decentralized fashion. While coordination can be achieved either way, the former approach is less prone to communication overheads and packet loss, does not suffer from synchronization issues and offers enhanced system-wide controllability, hence the focus of the present work on the centralized scheme. In line with this rationale, Dresner et al. \cite{v2iintersect} developed a protocol for multi-vehicle coordination via a centralized controller relying on {\em Vehicle-to-Intersection} (V2I) communication network~\cite{v2i}. The protocol employs a reservation mechanism wherein AVs declare their destinations to the controller and wait until permission is granted. The controller provides each vehicle with trajectory details, modeled as a path on an abstracted grid map, as well as destination speed. The work is extended in~\cite{au2010motion} with a refined trajectory representation considering vehicle kinematics. To avoid collisions, a grid-based collision map is constructed with inflated grid cells as a buffer. However, tuning cell size heuristically yields more conservative behavior when the cell is large and deteriorated performance when small, with no clear relationship to the actual probability of collision. Furthermore, the coordination mechanisms in~\cite{au2010motion,v2iintersect} rest on the {\em First-come-first-serve} (FCFS) principle, which \textit{lacks guarantees} on the global optimality of  the attained objective value (e.g., maximum throughput).

Different from existing centralized planners for optimizing intersection management (e.g., \cite{au2010motion,v2iintersect,zhao2021bilevel, ahn2016semi}), we develop a \textit{chance-constrained model} which alongside optimal control \textit{ensures} that collision probability remains within prescribed limits despite the imposed uncertainty. In planning under uncertainty, the Multi-agent Markov Decision Process (MMDP), which extends the classical MDP \cite{howard1960dynamic} to multiple agents, is a well-established mathematical formalism~\cite{boutilier1999sequential}. A special class of MDPs with non-negative utilities is known as the Stochastic Shortest Path (SSP) \cite{bertsekas1991analysis} (a.k.a. Stochastic Longest Path). Notoriously, MMDPs suffer from an exponential joint action space and a state space that is typically exponential in the number of agents. Notable research efforts have been directed towards exploiting agents' interactions to facilitate the problem's tractability (see, e.g., \cite{becker2004decentralized,spaan2008local}). Of particular interest, we consider MMDPs with {\em local interactions} \cite{melo2011decentralized,scharpff2016solving}, where agents have to coordinate their actions \textit{only at certain} interaction zones. We explore such an interaction scheme within an SSP framework with additional constraints that capture the probability of failure (e.g., collision) under several risk criteria. In particular, as exemplified in Fig.~\ref{fig:colide-points}, the studied intersection is modelled as a collection of finite ``interaction'' points where vehicles (i.e., agents) are likely to collide. However, contrary to prior methods which encode failure as a negative penalty \cite{scharpff2016solving,melo2011decentralized}, we bound the probability of collision by a preset threshold. This allows for a more versatile representation of safety requirements and has been adopted recently in the literature for the single AV scenario \cite{huang2018hybrid,huang2019online}. We remark that the solution techniques for MMDPs provided in \cite{melo2011decentralized,scharpff2016solving} are not amenable to these newly introduced constraints.

To further the design of safe and efficient traffic management programs, the present study proposes a risk-aware Intelligent Intersection system for AVs and HVs that, in spite of possible uncertainties, maximizes intersection's throughput without infringing the desired risk tolerance level. More concretely, the contributions and roadmap of this paper can be summarized as follows: 

    \begin{enumerate}
        \item In Sec.~\ref{sec:system}, we lay out the architecture of the proposed Intelligent Intersection framework where AVs are proctored by a centralized controller whereas HVs follow traffic light signals. As illustrated in Figs.~\ref{fig:test2} and~\ref{fig:intent}, the system is augmented with a \textit{probabilistic intent detector} as well as \textit{robust motion model generator} for HVs, allowing to cater for real-world nuances (e.g., imperfect trajectory estimation, ambiguity in drivers' decisions).  
        \item In Sec.~\ref{sec:ccssp}, we formulate a novel Multi-agent Chance-Constrained Stochastic Shortest Path (MCC-SSP) model with local interactions and devise \textit{an exact solution method} based on Integer Linear Programming (ILP). Through rigorous analytical scrutiny, the probabilistic constraints in MCC-SSP are proved reducible to equivalent linear ones in ILP (Theorem~\ref{lem1}). More importantly, the ILP formulation features \textit{polynomial number} of variables and constraints when the number of agents per interaction is small, thereby \textit{improving upon the state-of-the-art} designed for the single agent case~\cite{alyassi2021dual, sungkweon2021ccssp}.

        \item Drawing on MCC-SSP formalism, Sec.~\ref{sec:model} develops a conditional planner that enables AVs to react to potential contingencies (e.g., an unexpected maneuver from an HV). The planner supplies AVs with {\em safe} (w.r.t. permissible risk limit) contingency plans -- a maneuver for every possible scenario. A hybrid risk calculation method is employed, permitting the computations to be carried out offline in most cases. Subsequently, Sec.~\ref{sec:traj_optimization} models the reference trajectories via multi-variate Gaussian processes known as Probabilistic Flow Tubes (PFT)~\cite{pft} (see Fig.~\ref{fig:pft}) and presents a computationally  efficient means of estimating their collision probabilities.
        \item Lastly, Sec.~\ref{sec:experiment} validates the effectiveness and practicality of the proposed risk-aware intelligent intersection system through a series of simulations. Specifically, taking the classical grid problem as a case study we first demonstrate the \textit{scalability} of the introduced ILP formulation against the number of agents and horizons as well as its \textit{invariance} to the number of states. Next, we investigate the proposed planner's computational feasibility and contrast its performance with that of common approaches: FCFS and standard signalized scheme. As simulations indicate, the featured planner outperforms the two benchmarks by up to a factor of 2 in maximizing the intersection's throughput and supports rapid planning for multiple horizons ($> 1Hz$).   
        
    \end{enumerate}



\section{Related Work}\label{sec:related}

 
Parallel to the aforementioned centralized planners, a separate line of research has been devoted to developing decentralized intersection management schemes resting on {\em Vehicle-to-Vehicle} (V2V) communication. For instance, the works in~\cite{carlino2013auction, bashiri2017platoon} present game-theoretic decentralized approaches in the context of platooning so as to optimize traffic flow under several scenarios, including intersections. Chandra et.al.~\cite{chandra2022gameplan} propose a game-theoretic approach for unsignaled intersections, where priory is defined based on the driver's behavior (aggressive having higher priority).
Zhang et.al.~\cite{zhang2021priority} present a decentralized priority-based intersection system for cooperative AVs. The system incorporates three levels of planning; the first level is based on FCFS for vehicles at the intersection zone, the second considers AVs followed by emergency AVs arriving at the intersection (far zone) by scheduling their arrival time accordingly, and the third level is for the lowest priorities. However, unlike their centralized counterparts, decentralized schemes could inflict diminished system-wide controllability and efficiency, let alone communication overheads. In this work, we focus on the centralized case, with V2I communication, allowing the system to achieve better overall global optimality. 

An essential component of autonomous vehicle planning is uncertainty estimation in the environment. In particular, we are interested in tactical-decision making in the presence of uncertainty in HVs' intentions \cite{bandyopadhyay2013intention} and tracking error, where AVs may not be able to follow precisely a reference trajectory, mainly due to control uncertainty. Also, AVs from different vendors may run different controllers, hence have trajectory variations from the same reference trajectory. To cope with uncertainty in AVs' tactical decision-making, several works in the literature model the problem as Markov decision process variants. Brechtel et al. \cite{brechtel2014probabilistic} model the decision-making problem as a continuous state partially observable MDP (POMDP), where observation stochasticity resembles perception noise (e.g., LiDAR point cloud noise). Hong et al. \cite{sungkweon2021ccssp} model intention recognition of human-driven vehicles as {\em hidden} state variables in POMDP. 
Besides optimizing an objective function in POMDPs, \cite{huang2018hybrid,huang2019online} consider an additional chance constraint that bounds the collision probability below a safety threshold. 
Arguably, with a properly modeled state space, as we illustrate in this work, the driving problem can be modeled as a fully observable chance-constrained MDP, which is more scalable than the POMDP counterpart. 

 MDP \cite{howard1960dynamic} is a widely used model for planning under uncertainty. 
 Moreover, the problem has a dual linear programming (LP) formulation \cite{d1963probabilistic} which solves the problem as a minimum cost flow problem also known as the Stochastic Shortest Path (SSP) problem. 
 A special class of MDP optimizes an objective function while also bounding the probability of constraint violations is often called {\em chance-constrained} MDP (CC-MDP) \cite{dolgov2003approximating} which is an NP-Hard problem. To reduce the problem, an approximation of the chance constraint using Markov's inequality was proposed by  \cite{dolgov2003approximating}, effectively converting the problem into an MDP with a secondary {\em cost function}, called constrained MDP (C-MDP). Another approach \cite{de2017bounding} applies Hoeffding's inequality to improve the approximation. Both methods provide conservative policies with respect to safety thresholds. Exact methods for solving CC-MDP rely on those used for the partially observable MDPs (CC-POMDP) \cite{santana2016rao,khonji2019approximability}. However, even for a single agent, such methods suffer from scalability. They require full history enumeration in the worst case, which makes the solution space exponentially large with respect to the planning horizon \cite{sungkweon2021ccssp}. To the best of our knowledge, our technique, even for the single-agent case, is the first {\em exact} method for solving CC-MDPs that does not require history enumeration.  Also, our method extends the approximate method \cite{de2017bounding}, which considers independent agents with a shared risk budget, to situations where agents can interact at specific locations.


\section{Intelligent Intersection System} \label{sec:system}
 	The intersection system unders tudy considers a risk-aware architecture \cite{khonji2020risk} for both AVs and HVs, with the former following a centralized controller and HVs adhering to traffic signals. We adopt the protocol specified in \cite{v2iintersect} for AVs, where each vehicle declares a target destination to the coordinator via V2I communication. The coordinator is a computing unit installed on a road side unit equipped with communication and sensing modalities. Given vehicles' target destinations, the coordinator optimizes the intersection's overall performance and transmits a trajectory specification to each vehicle. The coordinator also predicts the intention of HVs and incorporates the uncertainty in the AV plans. Overall, the proposed system comprises six modules, as depicted in Fig. \ref{fig:high-level}, which are elaborated in the paragraphs to follow.

      \begin{figure}[!t]
    \centering 
    \includegraphics[width=\linewidth]{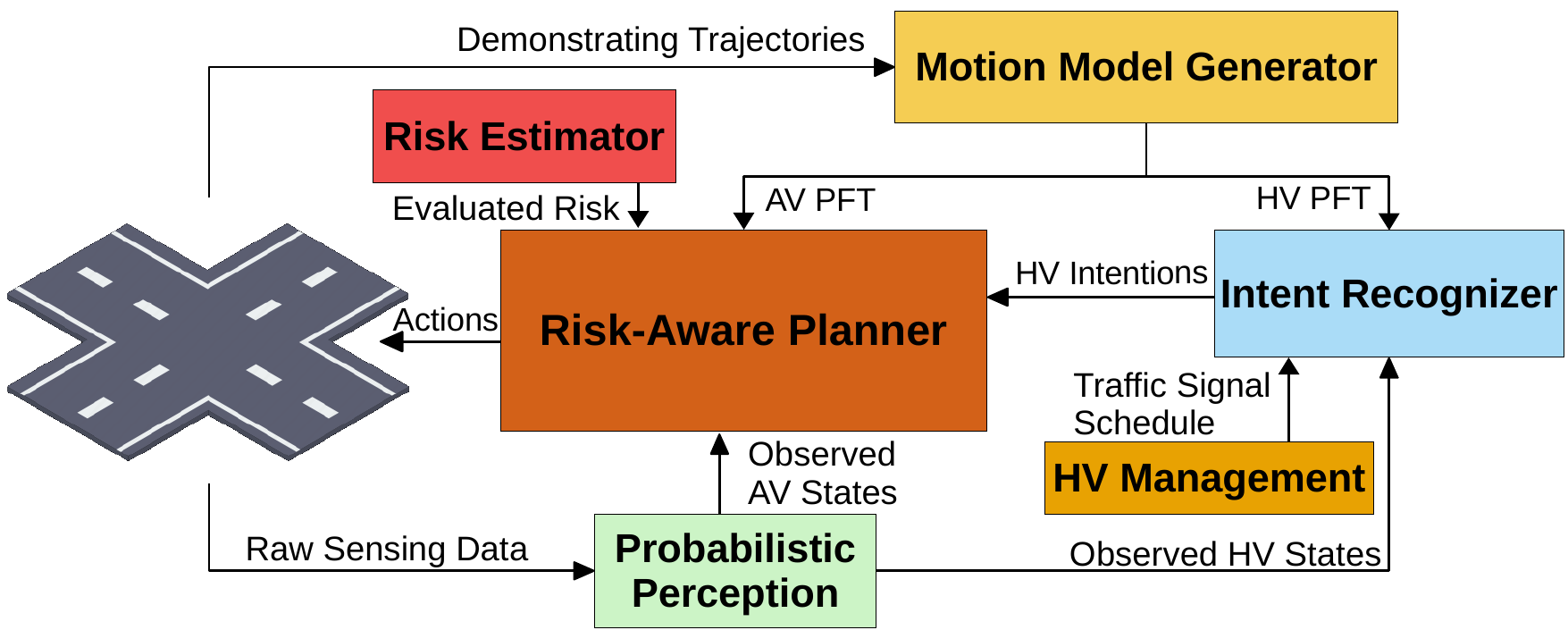}
 	\caption{High-level architecture of the featured Intelligent Intersection system.} 	\label{fig:high-level}
 	\end{figure}
    
    \vspace{2pt}
    \noindent{\textit{\bfseries Probabilistic Perception}:}
    Perception uncertainty can be specified as a distribution in which the object is located. One way to obtain a distribution over the object's shape and location is through point clouds. The points are segmented according to the face-plane of the obstacle they belong to; points belonging to the same face can be used to perform a Bayesian linear regression on the plane's parameters. Given a Gaussian prior on the face-plane parameters and under reasonable assumptions about the underlying noise-generation process, the posterior on the face-plane parameters will also be Gaussian \cite{axelrod2018provably}.
    
    A recent work in \cite{senanayake2018automorphing} presents a fast algorithm that generates a probabilistic occupancy model for dynamic obstacles in the scene with few sparse LIDAR measurements. Typically, the occupancy states exhibit highly nonlinear patterns that cannot be captured with a simple linear classification model. Therefore, deep learning models and kernel-based models can be considered as potential candidates.
    This module obtains a two-dimensional top-down representation of all vehicles in the scene along with uncertainty in their locations. For AVs, accurate location and pose (along with Gaussian uncertainty in location, which could be estimated via Kalman filter) can be transmitted through V2I~\cite{hartenstein2009vanet}. Moreover, infrastructure sensors, such as cameras, can obtain HVs poses and validate AV pose estimations via segmentation methods \cite{khan2013image}.

    \vspace{2pt}
    \noindent{\textit{\bfseries Motion Model Generator}:} The system continuously elicits trajectory traversal data of both AVs and HVs to learn probabilistic motion models in an online (real-time) fashion. These models are helpful in two ways: (i) for AVs, the learned motion models capture the uncertainty in the control system (ii) for HVs, the learned motion models represent how humans drive according to different driving patterns associated with tracking uncertainties.
    To encode uncertainty in trajectory traversal, we leverage a probabilistic representation, called Probabilistic Flow Tube (PFT) \cite{pft}, that encodes a nominal trajectory and uncertainties for a given driving pattern.
    The PFT is learned from a set of demonstrating trajectories, where each trajectory is composed of a sequence of positions. The output is a  Gaussian process, a sequence of means and covariances that represents a nominal trajectory and uncertainties in traversals. We refer to  \cite{pft} for further details on PFTs. The generation of the initial trajectory for AVs is described in Sec.~\ref{sec:traj_optimization}. 

    Existing works \cite{pft,huang2019online} assume the demonstrating trajectories are pre-labeled with maneuver types and learn a PFT for each maneuver type, which entails extra labeling effort and confines to a fixed taxonomy of maneuvers. Instead, we leverage an unsupervised clustering algorithm to create distinct clusters \cite{sung2012trajectory} given demonstrating trajectories representing different driving patterns and learn a distinct PFT for each cluster. As we continuously acquire more trajectories from the perception system, we update the clusters when the variance of a PFT is excessively high.    We visualize an example of PFT and its demonstrating trajectories in Figs.~\ref{fig:pft} and \ref{fig:test2}, respectively. Since the demonstrating trajectories are collected with different sizes, we use dynamic time warping (DTW) \cite{myers1980performance} to ensure equally sized trajectories prior to applying the PFT generator. 
 Then, PFTs are computed by sampling multiple trajectories, each obtained by following a nominal trajectory using a PID controller with slightly perturbed parameters to emulate manufacturer-specific controller deviations. 

    \vspace{2pt}
    \noindent{\textit{\bfseries HV Management}:}
    Considering that the intersection system lacks communication means with HVs, we assume HVs follow traffic signals and the employed intent recognition module predicts the trajectory of HVs to generate safe plans for AVs. For instance, if the traffic light is green on a given road, only HVs from that road may enter the intersection immediately. In contrast, AVs from any road may enter the intersection upon the system's command.

    \noindent{\textit{\bfseries Intent Recognizer}:} This module predicts human driver's intention as a distribution over a discrete set of candidate PFTs, thereby allowing AVs to maintain contingency plans for all potential scenarios. This is a departure from current approaches in the literature, where an AV is often provided with a single trajectory, whereas we provide a response trajectory \textit{for each potential scenario}. Contingency planning enables fast response to risky scenarios. The set of candidate PFTs are elected based on road policies that determine the set of legal maneuvers, each with its corresponding PFT.
    Intent recognition, in its own field, has been extensively studied and there are more sophisticated learning-based methods, such as \cite{huang2019uncertainty,ma2019trafficpredict}, that can produce accurate motion predictions conditioned on more detailed priors. 
    
    Given the tracked trajectory points of HVs, the module predicts the set of PFTs that the human driver is likely to follow through Bayesian filtering.
    First, we supplement the tracked vehicle trajectory with more data points based on polynomial fit, 
    then extend it into future horizons based on the polynomial coefficients to obtain an augmented trajectory. The augmentation allows us to provide sufficient data in case of short observation tracks.
    Second, we compute the observation probability of the augmented trajectory conditioned on each candidate PFT. The observation probability is multiplied by a prior distribution over PFT probabilities to obtain a posterior distribution of a candidate PFT the driver intends to follow. 
    The result is a distribution over a set of candidate PFTs, which allows us to leverage a risk-aware planner for each contingency, as explained in the following sections.
    
     Given a prior of all possible PFTs, the module computes the likelihood of the tracked trajectory against each PFT and updates the posterior distribution over discrete PFT choices. The future trajectory points are then predicted given the nominal trajectory and uncertainties associated with each PFT.

    \vspace{2pt}
    \noindent{\textit{\bfseries Risk-Aware Planner}:} 
    The planner takes as input the perception data, including observed AV states, PFT motion models learned for AVs, intentions of HVs, and outputs contingency plans for each AV in the scene. In Sec.~\ref{sec:model}, we formulate the intersection problem as an MCC-SSP and devise \textit{an exact} ILP-based solution approach.

\vspace{2pt}
\noindent{\textit{\bfseries Risk Estimator}:}
The online nature of the system and the high number of queries from the planner call for a computationally lightweight risk calculation method. Since PFTs are multi-dimensional Gaussian processes, collision at a given time amounts to solving multi-variate integrations. 
A straightforward alternative to the exact risk computation method is a high-resolution Monte Carlo sampling. Moreover, in Sec.~\ref{subsec:risk_precomp} we show how to compute the risk offline for most cases.

\section{Multi-Agent Chance-Constrained Stochastic Shortest Path} \label{sec:ccssp}
In this section, we formally define the fixed-horizon Multi-agent Chance-Constrained Stochastic Shortest Path (MCC-SSP) model.\footnote{Although the paper's results extend to multi-agent CC-MDP, we emphasize the stochastic shortest path variant as we believe that non-negative utility values fully capture relevant applications. We also want to discourage using negative values in the objective to penalize risk since chance constraints provide a more natural approach to risk representation.}

\subsection{Problem Definition}


\subsubsection{Agent Model}
We are given a set of agents $\cX$ (e.g., vehicles), each following a Markov decision process model 
$M^v=\langle \cS^v, \cA^v, T^v,U^v, s_0^v\rangle, v \in \cX,$ where
\noindent
\begin{itemize}
\item  $\cS^v$ and $\cA^v$ are finite sets of {states} and {actions} for agent $v$, respectively.
\item $T^v: \cS^v \times\cA^v\times \cS^v \rightarrow [0,1]$ is a probabilistic {transition function} between  states, 
$T^v(s^v,a^v,\bar s^{v}) = \Pr(\bar s^{v} \mid a^v,s^v)$, where $s^v,\bar s^{v} \in \cS^v$ and $a^v\in \cA^v$.
\item $U^v: \cS^v \times \cA^v\rightarrow \RR_+$ is a non-negative {utility function}.
\item $s^v_0$ is an initial state of agent $v\in \cX$.
\end{itemize}

\subsubsection{MCC-SSP Model}
The examined model considers situations where agents interact only at certain {\em interaction points} (see Fig.~\ref{fig:colide-points} for an illustration), indexed by a set $\cN$.  Each point $i\in \cN$ is in charge of coordination among multiple agents denoted by subset $\cX^i \subseteq \cX$.  We assume that every agent is assigned to at least one coordination point (i.e., $\cX = \cup_{i\in \cN} \cX^i$). 
Such interactions among agents entail risk of failure (e.g., collision). We consider multiple risk criteria, indexed by a set $\cJ$. 
Formally, we define the 
MCC-SSP problem as a tuple 
$M\triangleq\langle \cX, \cN, \cJ, (\cS^i, \cA^i, T^i,U^i, s_0^i)_{i\in \cN}, h, r^j(v,v')_{v,v' \in \cX}, (\Delta^j)_{j\in \cJ} \rangle$ where,
\begin{itemize}
\item $\cX$ is the set of agents, $\cN$ is the set of interaction points, and $\cJ$ is the set of risk criteria.
$\cS^i\triangleq \times_{v\in \cX^i} \cS^v$ is a factored set of {\em interaction states} of the coordinated agents, and $\cA^i\triangleq\times_{v\in \cX^i} \cA^v$ is a factored set of joint {actions} for interaction $i\in\cN$, respectively. For $a^i\in\cA^i$, we write $a^i_v$ to denote the action of vehicle $v\in \cX^i$.
	\item
$T^i: \cS^i \times\cA^i\times \cS^i \rightarrow [0,1]$ is the joint transition function such that $T(s^i,a,\bar s^{i}) \triangleq \prod_{v\in\cX^i} T^v(s^v, a^v,{\bar s^v})$, where $\bar s$ is the next state. 
	\item
$U^i: \cS^i \times \cA^i\rightarrow \RR$ is the total utility such that $U^i(s^i, a^i)\triangleq \sum_{v \in \cX^i} U^v(s^v, a^v).$ When an agent belongs to multiple interaction points, then we count utility of that agent at one interaction only. In other words, if $v \in \bigcap_i \cX^i$ then there is only one interaction $i'$ such that $U^{i'}(s^{i'},a^{i'})\triangleq \sum_{v' \in \cX^{i'}} U^{v'}(s^{v'}, a^{v'})$ and the rest will have
$U^i(s^i,a^i)\triangleq \sum_{v' \in \cX^i\backslash \{v\}} U^{v'}(s^{v'}, a^{v'}).$	
	\item
$s^i_0\triangleq (s^v_0)_{v \in \cX^i} $ is the joint initial state.
	\item $h$ is the planning horizon. 
\item $r^j(v,v'): \cS^v\times \cS^{v'}  \rightarrow [0,1]$ provides the probability of failure (e.g., collision) due to interaction between agent $v$ and $v'$   at their respective states according to risk criterion $j$; and $\Delta^j$ is the corresponding risk budget, a threshold on the probability of failure over the planning horizon.
\end{itemize}
We represent the joint actions of all interaction points by $A \triangleq \times_{v \in \cX} \cA^v$, and joint states by $\cS \triangleq \times_v S^v.$ For convenience, we write $U(s, a)\triangleq \sum_{v \in \cX} U^v(s^v, a^v)$ to denote to the total utility of state $s\in \cS$ and action $a \in \cA$. 
A  {\em deterministic} and {\em non-stationary} policy $\pi(\cdot,\cdot)$ is a function that maps a state and time step into an action, $\pi:\cS\times \{0,1,...,h-1\} \rightarrow \cA$. A {\em stochastic} policy $\pi: \cS \times \{0,1,...,h-1\}\times \cA \rightarrow [0,1]$ is a probability distribution over actions from a given state and time.\footnote{To avoid clutter, we write $\pi(s_k, a)$ instead of $\pi(s_k, k,a)$ for state $s_k$.} 
We write  $\pi^i: \cS^i \rightarrow \cA^i$ to encode the joint action of agents $\cX^i$ at interaction $i\in \cN$, as per a feasible policy $\pi$. 
A {\em run} is a sequence of random joint states  $S_0, S_1,\ldots, S_{h-1}, S_h$ resulting from policy execution, where $S_0 = s_0\triangleq(s^v_0)_{v\in\cX}$ is known. We use superscript $i$ to denote the corresponding run with respect to interaction $i\in \cN$ (similarly we use superscript $v$ to that of agent $v$). 
Let $R^{j}(s^v, s^{v'})$ be a Bernoulli random variable for failure between $s^v\in \cS^v$ and $s^{v'} \in \cS^{v'}$ with respect to criterion $j\in \cJ$. As per MCC-SSP model $M$,  $\Pr(R^{j}(s^v, s^{v'}) = 1) = r^j(s^v, s^{v'})$.
The objective of {\sc MCC-SSP} is to compute a policy (or a conditional plan) $\pi$ that maximizes the cumulative expected utility (or minimizes the cumulative expected cost) while maintaining risks below  the given thresholds $\Delta^j$. Formally,
{\small{\begin{align}
\small{\textsc{(MCC-SSP)}} & \nonumber\\
\max_{\pi}& ~\mathbb{E} \Big[ \sum_{\mathclap{t=0}}^{h-1} U(S_{t}, \pi(S_{t}) )  \Big] \nonumber \\
\text{s.t.}&~ \Pr\Big(\bigvee_{t=0}^{h} \bigvee_{v, v' \in \cX}R^{j}(S^v_t, S^{v'}_t) \mid \pi \Big) \le \Delta^j, j \in \cJ. \label{con1}
\end{align}}}

According to the definition above, if $\cX= \cN$, i.e., single-agent interaction points, then agents are independent, except for sharing the risk budget. Such variant has been studied extensively in the literature for constrained MDPs (see \cite{de2021constrained} for a comprehensive survey). As reported below, we provide an exact computation method, which improves upon the approximate approach provided in \cite{de2017bounding}.\footnote{We note that the chance constraint in \cite{de2017bounding} is slightly more general than ours, as it bounds the total probability of a sum of costs exceeding a certain threshold. The current approach can capture such constraints by augmenting the state space to include all possible distinct values of the sum. Arguably, in some applications, the number of distinct values could be exceedingly large. Therefore, we can discretize the set of values such that, in the worst case, we violate the constraint by at most a factor of $(1+\epsilon),$ often referred as resource augmentation model.}  

\subsection{Execution Risk} 

Define the {\em execution risk} of a run at joint state $s_k$  as 
$\textsc{Er}^j(s_k) \triangleq   \Pr\Big(\bigvee_{t=k}^h \bigvee_{v, v' \in \cX}R^{j}(S^v_t, S^{v'}_t) \mid S_k = s_k\Big).$ By definition, Cons.~\raf{con1} is equivalent to $\textsc{Er}^j(s_0) \le \Delta^j$. Here, it is assumed that any pair of agents may fail in at most one interaction point. This assumption will be important to obtain a {\em linear} constraint and holds for the studied intersection application (see Fig.~\ref{fig:colide-points}) where each collision point is between a unique set of agents. For applications where such assumption may not hold, Eqn.~\raf{eq:bool} below establishes an upper bound on the execution risk based on the Union bound. That being so, the proposed approach in Sec.~\ref{sec:ilp} can still generate a feasible yet possibly suboptimal solution.
The execution risk can be written as
\begin{flalign}
\textsc{Er}^j(s_k)=&\Pr\Big(\bigvee_{t=k}^h \bigvee_{i\in \cN} \bigvee_{v, v' \in \cX^i}R^{j}(S^v_t, S^{v'}_t) \mid S_k = s_k\Big)&& \notag\\%
=&\sum_{i \in \cN}   \Pr\Big(\bigvee_{t=k}^h \bigvee_{v, v' \in \cX^i}R^{j}(S^v_t, S^{v'}_t) \mid S^i_k = s^i_k\Big),\label{eq:bool}&& \raisetag{25pt}
\end{flalign}\noindent where the last equation holds by the assumption on mutual exclusivity of events and by conditional independence. 
We  write execution risk at interaction point $i$ as  $\textsc{Er}^j(s^i_k) \triangleq  \Pr\Big(\bigvee_{t=k}^h \bigvee_{v, v' \in \cX^i}R^{j}(S^v_t, S^{v'}_t) \mid S^i_k = s^i_k\Big)$. Thus, $\textsc{Er}^j(s_k) = \sum_{i \in \cN} \textsc{Er}^j(s^i_k)$. 
\begin{lemma}\label{lem:er}
The execution risk at interaction point $i\in \cN$ can be written recursively as
{\footnotesize{\begin{align}
\textsc{Er}^j(s^i_k)=  \sum_{s^i_{k+1}\in \cS^i}\sum_{a^i \in \cA^i} \textsc{Er}^j(s^i_{k+1}) \pi(s^i_k, a^i) \widetilde T^j(s^i_k, a^i, s^i_{k+1})& 
 + \widetilde r^j(s_k),&  \notag
\end{align}}}\noindent where $\widetilde r^j(s^i_k)\triangleq 1 - \prod_{v,v' \in \cX^i}(1-r^{j}(s^v_k, s^{v'}_k))$ is the probability of failure at $s^i_k$, and $\widetilde T^{i,j}(s^i_k, \pi(s^i_k), s^i_{k+1})\triangleq  T^i(s^i_k, \pi(s^i_k), s^i_{k+1})\prod_{v,v' \in \cX^i}(1-r^{j}(s^v_k, s^{v'}_k))$.
\end{lemma}
\begin{proof}
See Sec.~\ref{lemproof} in the Appendix.
\end{proof}

\subsection{Integer Linear Programming Formulation}\label{sec:ilp}

Define a variable $x_{s,k,a}^{i,j}  \in [0,1]$ for each state $s^i_k$ at time $k$, action $a^i\in \cA^i$, agent $i\in\cN$, and constraint $j \in \cJ$ such that
\begin{align}
&\sum_{a^i \in \cA^i}  x_{s,k,a}^{i,j} = \sum_{s^i_{k-1} \in \cS^i} \sum_{a^i \in \cA^i} x^{i,j}_{s,k-1,a} \widetilde T^{i,j}(s^i_{k-1}, a^i, s^i_{k}), \notag \\
& \hspace{10mm} k = 1,...,h-1,  s^i_k \in \cS^i, i \in \cN, j \in \cJ\cup \{0\}, \label{conf1}\\
&\sum_{a^i \in \cA^i} x^{i,j}_{s,0,a}  = 1, \quad i\in\cN, \label{conf2}
\end{align}
where $\widetilde T^{i,0}(\cdot,\cdot, \cdot) \triangleq T^i(\cdot,\cdot, \cdot)$. 
As such, the above {\em flow} equations for $j=0$ represent the standard dual-space constraints for SSP~\cite{10.5555/2074158.2074203}. In the context of SSP, $x^{i,0}_{s,k,a}$ stand for the probability of agent $i$ taking action $a^i$ from state $s^i_k$.

By recursively expanding the execution risk at $s_0$ using Lemma~\ref{lem:er} and Eqn.~\raf{eq:bool} we arrive at the following result.
\begin{theorem}\label{lem1}
Given a conditional plan $\bx$ that satisfies Eqn.~\raf{conf1}-\raf{conf2}, the execution risk can be written as a linear function of $\bx$,
{\scriptsize\begin{align}
\textsc{Er}^j(s_0) =  \sum_{k = 1}^{h}\sum_{i\in\cN~~} \sum_{\mathclap{\substack{s^i_{k-1}\in \cS^i \\a^i \in \cA^i, s^i_k\in \cS^i }}} \widetilde r^{j}(s^i_{k}) x^{i,j}_{s,k-1,a} \widetilde T^{i,j}(s^i_{k-1}, a^i, s^i_{k})
+\sum_{i \in \cN} \widetilde r^j(s^i_0).\nonumber
\end{align}} 
\end{theorem}
\begin{proof}
See Sec.~\ref{thmproof} in the Appendix.
\end{proof}

Consequently, we can formulate MCC-SSP as an ILP that has a polynomial number of variables and constraints in terms of $h,|\cN|, |\cX|, |\cA^v|$ when the number of agents per interaction is at most a constant $c$, i.e., $|\cX^i|\le c$. This would only require enumerating agent actions within each interaction point which is significantly lower than the complete enumeration.
{\footnotesize\begin{align}
\max_{x, z}~~& \sum_{k=0}^{h-1}\sum_{i \in \cN}\sum_{s^i_k \in \cS^i, a^i \in \cA^i} x^{i,0}_{s,k,a} U(s^i_k, a^i) \qquad \qquad \textsf{(MCC-SSP-ILP)}\notag\\
\text{s.t.}~~& \text{Cons.~\raf{conf1}-\raf{conf2}}, \notag\\
&\sum_{k = 0}^{h-1}\sum_{i\in \cN} \sum_{\mathclap{\substack{~~~s^i_{k}\in \cS^i\\~~~ a^i \in \cA^i, s^i_{k+1}\in \cS^i }}} \widetilde r^{j}(s^i_{k+1}) x^{i,j}_{s,k,a} \widetilde T^{i,j}(s^i_k, a^i, s^i_{k+1})\le \widetilde\Delta^{j}, j\in \cJ \label{riskcons} \\
&\sum_{a^i\in \cA} z^i_{s,k,a}\le 1, \quad k=0,...,h-1, s^i_k \in \cS^i, i \in \cN\label{conzone}\\
& x^{i,j}_{s,k,a} \le z^i_{s,k,a}, \quad   i\in \cN, j \in \cJ, k=0,...,h-1, s^i_k \in \cS^i\label{conzbind}\\
&\sum\limits_{\mathclap{a^i \in \cA^i \mid a^i_v = \overline a^v}} z^{i}_{s,k,a} = \sum_{\mathclap{~~~~{a^i}' \in \cA^{i'}\mid a^{i'}_v = \overline a^v}} z^{i'}_{s,k,a},  \ v\in \cX,  \overline a^v \in \cA^v, i, i' \in \cN \mid v \in \cX^i \cap \cX^{i'}    \label{con:bind}\\
&z^i_{s,k,a}\in \{0,1\}, \quad x^{i,j}_{s,k,a} \in [0,1],~~\quad i \in \cN, j \in \cJ, \notag\\
& s^i_k \in \cS^i,  k=0,...,h-1.
\end{align}} 

In {\sc MCC-SSP-ILP}, Cons.~\raf{riskcons} follows directly from Theorem~\ref{lem1}, where $\widetilde\Delta^{j}\triangleq \Delta^j - \sum_{i\in \cN} \widetilde r^j(s^i_0).$ 
The variable $z_{s,k,a}^i$ is used to bind the actions of $i$ across all flows. In other words, if action $a^i$ is selected with respect to risk criterion $j$, and $a'$ for criterion $j'$, then $a^i=a'$.  Thus, Cons.~\raf{conzbind} ensures that the same action is selected. Since, $z_{s,k,a}\in \{0,1\}$, Cons.~\raf{conzone} guarantees at most one action is chosen at each node. Lastly, Cons.~\raf{con:bind} maintains the consistency of the selected action for each agent across all interaction points.

To quantify the planner's complexity, we next appraise the worst-case running time of MCC-SSP analytically by examining the maximum number of nodes in the solution space. As an upper bound, consider a tree-based structure that expands all the nodes for a defined horizon of $h$ without combining similar states as in the graph structure that the MCC-SSP follows. Given a set of $|\cN|$ interaction points each containing $|\cX^i|$ agents, and $|\cA^v|$ number of actions per agent $v$, we have $\prod_{v\in \cX^i}|\cA^v|$ possible actions per interaction point $i \in \cN$. Thus, the last level of the solution tree for interaction point $i$ contains $(\prod_{v\in \cX^i}|\cA^v|)^h$ nodes. Hence, The total number of nodes for the planning problem is $O(\prod_{i\in \cN}((\prod_{v\in \cX^i}|\cA^v|)^h))$, which can be also written as $O(|\cN|((|\cX^{i^m}|\cdot|\cA^{v^m}|)^h))$ where $i^m$ is the point with maximum agents, and $v^m$ is the agent with maximum actions. 

\section{Risk-Aware Planning Under MCC-SSP}\label{sec:model}

\subsection{Agents and Interaction Points}
Recall that in the MCC-SSP formulation vehicles are indexed by the set $\cX$, wherein HVs constitute a subset $\cY\subset\cX$ and have a single action (hence are uncontrollable). The set of interaction points $\cN$ represents all possible collision points between agents in the intersection, including those right before the intersection, as pictured in Fig.~\ref{fig:colide-points}.
Each interaction point $i \in \cN$ is associated with a set of {\em reference} maneuvers $\cM^i$ (green dashed lines in Fig.~\ref{fig:colide-points}). A reference maneuver $m_v\in\cM^i$ is defined as a path that an agent $v$ can follow regardless of the speed, while actions are defined as trajectories. 

\subsection{Action Model}

Action $a^i\triangleq(a^i_v)_{v\in\cX^i}$ at the interaction point $i \in \cN$ resembles one possible combination of variants of reference maneuvers that pass through that point, with $a^i_v\in \cA^v$ representing a variation of a corresponding reference maneuver $m_v$ with a specific speed.
For an agent, a typical scenario would constitute a two-action model: 1) perform the maneuver or 2) wait. We \textit{expand} the scope by introducing \textit{speed-sensitive maneuvers}, such as {\tt turn\_left\_slow, turn\_left\_fast, wait}.

\subsection{State Representation}\label{sec:staterpr}
An interaction point $i \in \cN$ at time $k$ is associated with a state $s^i_k\triangleq (k,p^v)_{v \in \cX^i}$, where $p^v$ is the PFT of the current executing maneuver of vehicle $v$ or the {\tt wait} command. Upon arriving at the intersection, the vehicle is added to the state with the {\tt wait} maneuver. 
The maneuver state changes (under a transition function) when an action is applied.
The vehicle's position and velocity are computed based on the progression of the maneuver PFT. We define the risk $r^j(s^i)$ of state $s^i$ as the probability of collision between the agents in the state (formalised in Sec.~\ref{sec:riskc}).

\subsection{Transition Function}
The transition function $T^v(s^v,a^v,\overline s^v)$ is computed by utilizing the intent recognition subsystem. The corresponding PFT resulting from action $a^v$ can be computed based on prior data collected by the Motion Model generator, as well as other external factors such as road conditions~\cite{huang2018hybrid,huang2019online, pft}. For HVs, the  Intent Recognizer can capture human uncertainty. We refer the reader to \cite{gindele2013learning} for more details. \textcolor{black}{Moreover, any significant deviation of an AV or HV from its trajectory (mainly taking an illegal turn) due to some fault or miscommunication is handled by setting the system to a halt state. A halt state causes all vehicles and light signals to stop until the vehicle clears the intersection.}


\vspace{-5pt}
\subsection{Operating Horizon}
We adopt a receding horizon approach for online execution. For each time step $t$, we solve an MCC-SSP model for a horizon of length $h$. Vehicles that leave the intersection are removed from $\cX$,  and the model is subsequently solved on a rolling basis for the remaining vehicles still in the intersection.
Agents that are still executing their maneuvers from the previous state are considered obstacles and thus, there are no actions that apply to them.
The horizon duration $\Delta t$ is chosen to be less than the action termination time $\tau$ in order to provide a smoother transition. However, it's noteworthy that a smaller horizon window will result in a shorter total planning time ($h \cdot \Delta t $). 

\subsection{Objective Function}\label{sec:objjj}
The proposed Intelligent Intersection system seeks to optimize the following \textit{multi-criteria} objective function: 1) Maximize the rate of flow (vehicles per unit time); 2) Minimize the maximum waiting time (duration from the moment the vehicle reaches the intersection to the moment it enters the intersection) to ensure fair distribution of traffic; 3) Facilitate a priority-based objective for emergency vehicles. 
Recall that the optimal policy $\pi^* = \arg\max_{\pi}  \mathbb{E}\Big[\sum_{t=0}^{h-1} U(s_{t}, a_t)  \mid \pi \Big]$ requires a definition of a utility function. For a state $s=(k,p_v)_{v\in \cX}$, and action $a=(a^v)_{v\in\cX}$, we define $U(s,a) \triangleq \sum_{v \in \cX} (U^v(a^v) \mid  a^v \neq \texttt{wait}), $ such that $U^v(a^v) \triangleq \lambda_0 {\texttt{vel}(a^v)}  +\lambda_1 P^v + \lambda_2 \sqrt{w^v}+ \lambda_3\frac{\sum_{v' \in \texttt{lane}(v)} P^{v'}}{|\texttt{lane}(v)|}$,
where $\lambda_0 \text{ to } \lambda_3 \in \RR$ are used to weigh each term. Here, $\texttt{vel}(\cdot)$ provides the velocity of a given reference trajectory (faster maneuvers provide higher rewards), $P^v$ represents a priority score range (e.g., from 1 to 10) and the last term is the average priority of the vehicle's lane, with $\texttt{lane}(v)$ standing for the set of vehicles that are assigned to vehicle $v$'s source lane. In the definition of $U^v(a^v)$,  $w^v$ captures the waiting time associated with the state.  This requires augmenting the state space to include total waiting time for each vehicle $v$, which increments whenever state maneuver $s^v = \texttt{wait}$. For exposition clarity, we omitted the waiting time from the state representation in Sec.~\ref{sec:staterpr}. 
Though sufficient for current purposes, the presented objective function can be further extended to incorporate progress (or speed) and comfort requirements for each agent~\cite{wei2014behavioral}.

\subsection{Risk Computation }
\label{sec:riskc}

Given that maneuvers are represented as PFTs, i.e. multivariate Gaussian processes, the risk can be calculated via Monte Carlo sampling. Particularly, to compute the risk between two maneuver trajectories $p^v, p^{v'}$, we extrapolate over time (over a duration of $\tau$ ), but at higher resolution (i.e., PFT resolution). Define $\mathtt{Col^{vv'}}(t)$ to be a Bernoulli random variable such that $\mathtt{Col}^{vv'}(t)=1$ if and only if collision occurs at time $t\in \{0, \Delta \tau, 2 \Delta \tau,..., \tau \}$ following their corresponding PFTs, and $\mathtt{Col}^{vv'}(t)=0$ otherwise. Then, at a given time $t$, the Monte Carlo sampling approach can be invoked to determine $\mathtt{Col}^{vv'}(t)$. Given $\mathtt{Col}^{vv'}(t)$ for $\forall t$, the risk of collision throughout the time horizon $\tau$ between agents $v$ and $v'$ is defined as $r^j(v,v')\triangleq\Pr\Big(\bigvee_{t=1}^\tau \mathtt{Col}^{vv'}(t) \Big)  = 1 - \Pr\Big(\bigwedge_{t=0}^\tau \neg \mathtt{Col}^{vv'}(t) \Big) = 1 - \prod_{t=0}^\tau \Pr(\neg \mathtt{Col}^{vv'}(t)) =1 - \prod_{t=0}^\tau (1-  \Pr( \mathtt{Col}^{vv'}(t)))$.

\subsection{Offline Risk Computation}\label{subsec:risk_precomp}
To streamline the planning process, we precompute the risk of all the possible states of every interaction point. For an interaction point $i$ with $|\cM^i|$ maneuvers, we allocate an $|\cM^i|-$dimensional lookup table of size $O(\tau'^{|\cM^i|})$, where $\tau' \geq \tau$ is the maximum number of potential progressions of a vehicle through the reference maneuver ($ \tau(m_v)=|m_v| \quad \forall \; m_v\in \cM^i$), which is dependent on the defined PFT time-step of the trajectory. After precomputing the risk for all collision points, we can efficiently obtain the risk of any state by retrieving the risk of every combination of vehicles in the state. We precompute the risk for multiple standard vehicle dimensions to generalize for any vehicle model.

\section{Trajectory Optimization for AVs} \label{sec:traj_optimization}
AVs in the intersection problem require a reference trajectory (action) to follow. Such trajectory should be optimal with respect to certain objectives (e.g., comfort) while respecting the vehicles' control limits. Moreover, the trajectory should be safe concerning static obstacles as well as other vehicles in the intersection (dynamic obstacles). 
To generate a set of reference trajectories, we employ a technique which finds the optimal path given the system and collision avoidance constraints.
We model the trajectory optimization task as a finite-horizon discrete-time shooting problem \cite{mayne1966second}, and tackle it with the Control-Limited Differential Dynamic Programming (DDP) algorithm \cite{tassa2014control}. In the subsections to follow, we address each of the above aspects.


\subsection{Dynamics Model}
The adopted vehicle dynamics model is based on the kinematic bicycle model. The bicycle model provides an approximate yet efficient representation (by averaging the speed of both wheels on a given axle) of the vehicle dynamics. The dynamics model with the center of mass as a reference point is defined as:
\begin{gather}
 \Dot{x}^c = v^c \cdot \cos(\theta+\beta) ~,~  \Dot{y}^c = v^c \cdot \sin(\theta+\beta) \nonumber\\
 \Dot{\theta} = \omega = v^c \cdot \frac{\tan(\zeta) \cos(\beta)}{L} ~,~  \beta = \tan^{-1}(l_r \frac{\tan(\zeta)}{L}),\nonumber
\end{gather}where $(x^c, y^c, \theta)$ is the position and heading of the vehicle, $v^c$ is the vehicle's velocity, $\zeta$ is the steering angle, $L$ is the length of the vehicle, and $l_r$ is the distance from the back axle to the center of mass. The dot notation (e.g., $\Dot{\theta}$) represents the derivative of the variable.
The trajectory state is defined as $\tilde{x} = [x^c, y^c, \theta, \zeta, v^c]$ and the control is defined as $\tilde{u} = [a, \Dot{\zeta}]$ where $a$ is the vehicle's acceleration (i.e.,  $\Dot{v^c}$).

We convert the continuous dynamics model into a discrete one by updating the state every $\Delta t$ using an explicit Euler integration scheme. The subsequent state $\tilde{x}_{k+1}$ is defined as:
\begin{gather}
x_{k+1}^c = x_k^c + \Dot{x}^c \cdot \Delta t ~,~
y_{k+1}^c = y_k^c + \Dot{y}^c \cdot \Delta t \nonumber\\
\theta_{k+1} = \theta_k + \Dot{\theta} \cdot \Delta t ~,~
\zeta_{k+1} = \zeta_k + \Dot{\zeta} \cdot \Delta t, v^c_{k+1} = v^c_k + a \cdot \Delta t.\nonumber
\end{gather}

\subsection{Cost Function}
The purpose of the cost function is to yield a smooth trajectory, while also imposing safety requirements to bypass static and dynamic obstacles. The integral cost function is defined as $[\Dot{\zeta}, a_c, c_s, c_d]$ where $a_c$ is the rotational acceleration ($a_c = (v^c)^2 \cdot \frac{\tan(\zeta)\cos(\beta)}{L}$), and $c_s$, $c_d$ are the static and dynamic obstacle costs as defined in Secs.~\ref{stt} and \ref{dnn}, respectively. A quadratic barrier function is used to enforce the limits on steering angle $\zeta$  (the rate of change of the steering angle is handled by the solver as an input constraint). The boundary cost function is $[x^c-G_x, y^c-G_y, \theta-G_\theta, v^c-G_v]$ where $G$ is the goal state.

\subsection{Static Obstacles}\label{stt}
In the studied intersection, curbs or unpaved areas are treated as static obstacles. Common obstacle avoidance methods either suffer from local minima (e.g., Artificial Potential Fields \cite{khatib1986real}) or tend to direct the state away from obstacles as much as possible (e.g., Harmonic Potential Fields \cite{kim1992real}). This served as a motivation to resort to the Signed Distance Field (SDF) \cite{finean2021predicted} approach with the hinge loss function \cite{mukadam2018continuous}. In implementing SDF, we first generate a binary array of the obstacles (Fig.~\ref{fig:sdf2}) based on the intersection map (Fig.~\ref{fig:sdf1}). Next, we apply the signed distance field function, which returns the signed distance $D_s$ from the zero contour in an array (Fig.~\ref{fig:sdf3}). Finally, we apply the hinge loss function (Fig.~\ref{fig:sdf4}), which returns a zero value if the state is not near an obstacle. 
The hinge loss function is defined as
$$
h(D_s)=
    \begin{cases}
        -D_s+\epsilon & \text{if } d \leq \epsilon\\
        \quad\;\; 0 & \text{if } d > \epsilon
    \end{cases},\notag
$$
where $\epsilon$ is a safety distance from the boundary of the obstacles. We set the static obstacle cost as $c_s = e^{-\frac{1}{2}h(D_s)^2}$.

\subsection{Dynamic Obstacles}\label{dnn}
We render other non-colliding trajectories 
that do not share a collision point but might collide if they were generated close to each other as dynamic obstacles. Two factors are accounted for in the dynamic obstacle avoidance, namely the vehicle's geometric shape and the expected controller uncertainty when following a trajectory.

As illustrated in Fig.~\ref{fig:vehicle_circles}, vehicle's shape can be captured by three adjacently placed overlapping circles. To determine whether two vehicles would collide, one can compare the Euclidean distance between center-points and sum of radius of any pair of circles among the two vehicles.

\begin{figure}[!b]
	\begin{subfigure}{.5\linewidth}
		\centering
		\includegraphics[width=\linewidth]{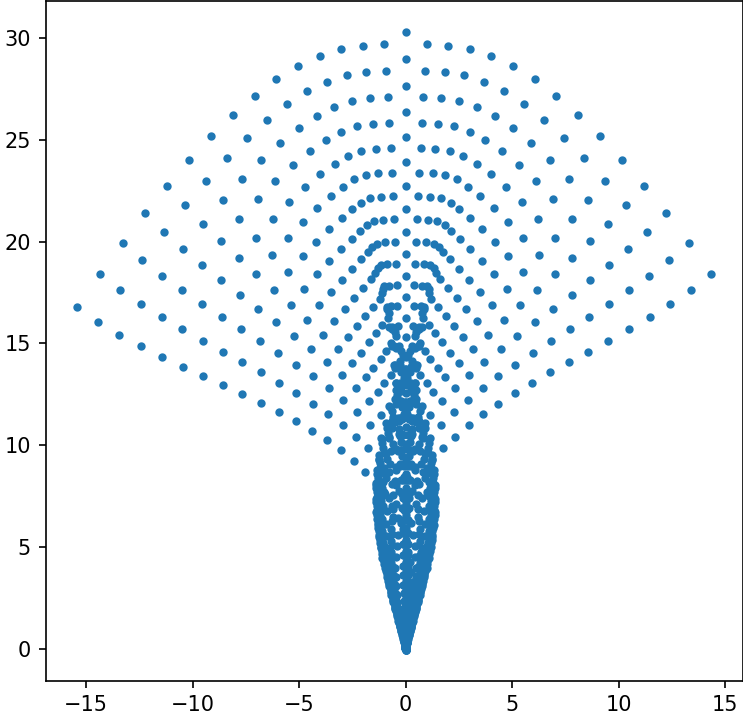}
		\caption{}
			\label{fig:reg_model_paths}
	\end{subfigure}
	\hspace{.03\linewidth}
	\begin{subfigure}{.43\linewidth}
		\centering \vspace{2.5mm}
		\includegraphics[width=\linewidth]{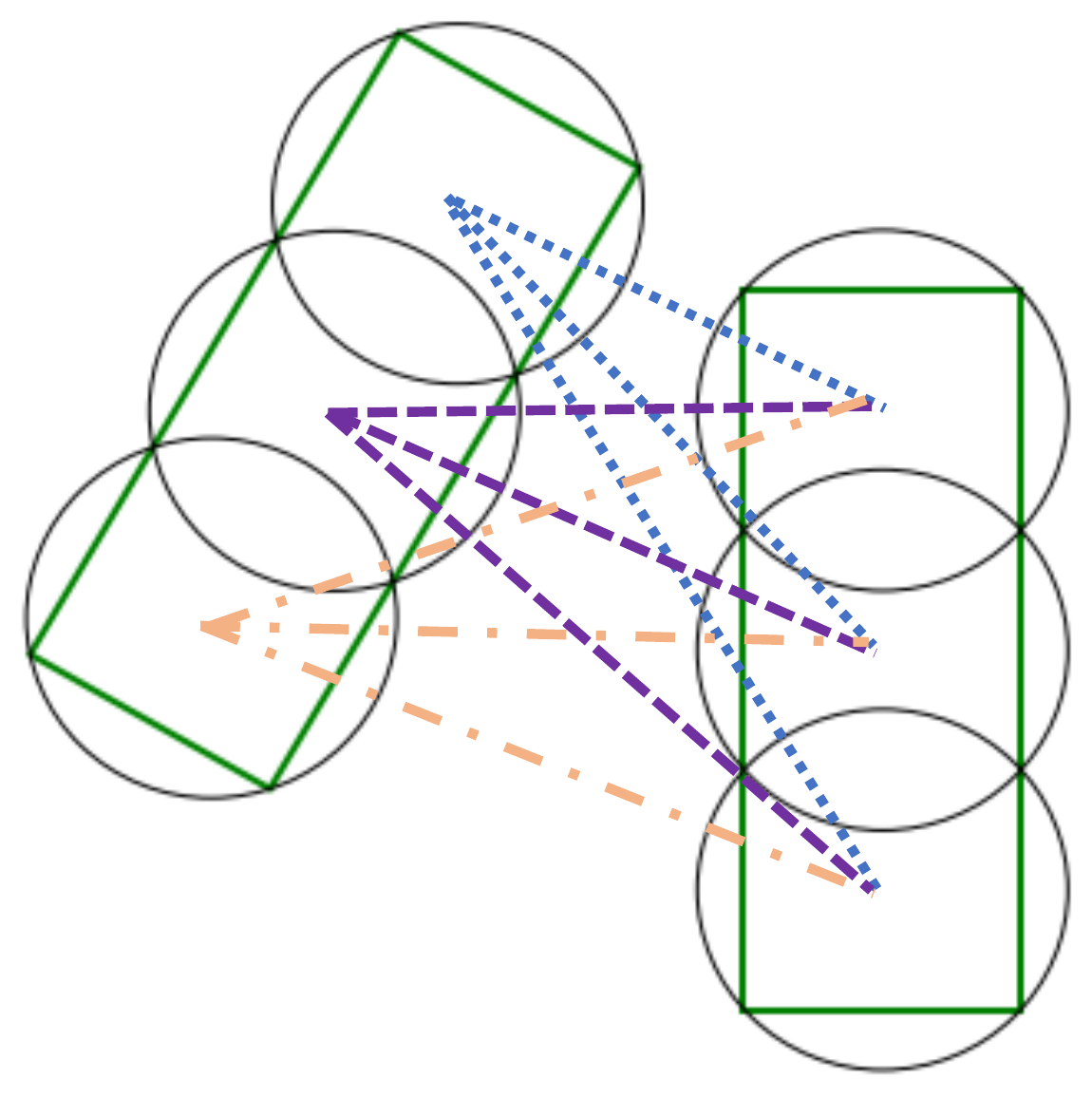}\vspace{1mm}
		\caption{}
		 	\label{fig:vehicle_circles}
	\end{subfigure}
	\caption{(a) A set of 30 uniformly distributed trajectories, and (b) the distance between two vehicles each represented via three circles. }
\end{figure}

\begin{figure*}[t!]
	\begin{subfigure}{.23\textwidth}
		\centering
		\includegraphics[width=\textwidth]{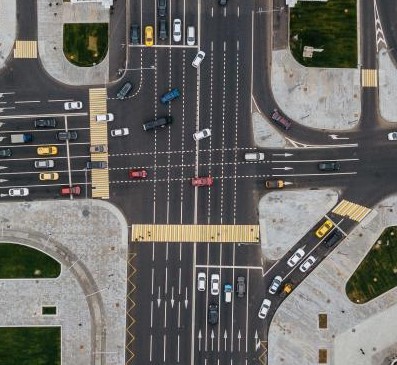}
		\caption{}
		\label{fig:sdf1}
	\end{subfigure}
	\begin{subfigure}{.23\textwidth}
		\centering
		\includegraphics[width=\textwidth]{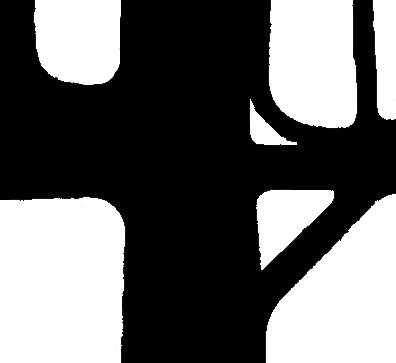}
		\caption{}
		\label{fig:sdf2}
	\end{subfigure}
	\begin{subfigure}{.26\textwidth}
		\centering
		\includegraphics[width=\textwidth]{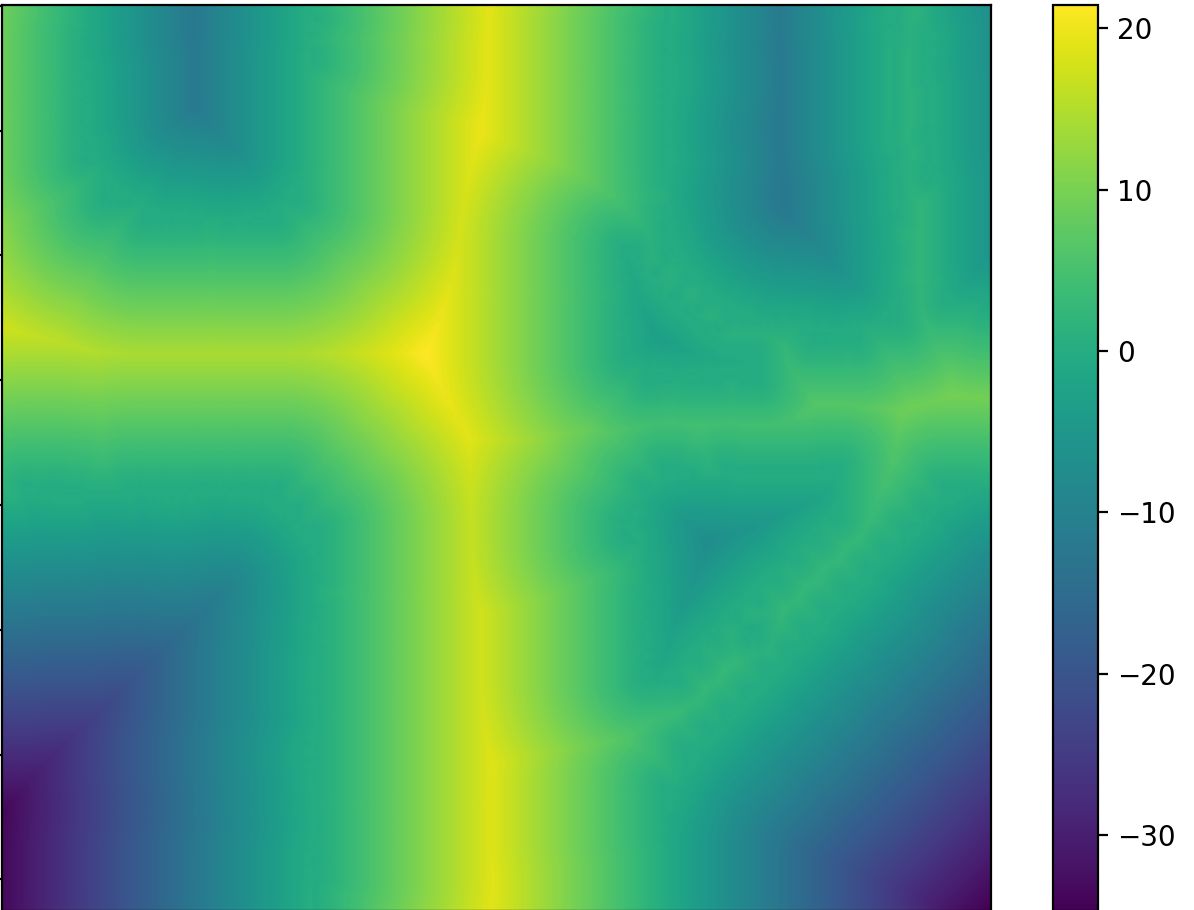}
		\caption{}\label{fig:sdf3}
	\end{subfigure}
	\begin{subfigure}{.26\textwidth}
		\centering
		\includegraphics[width=\textwidth]{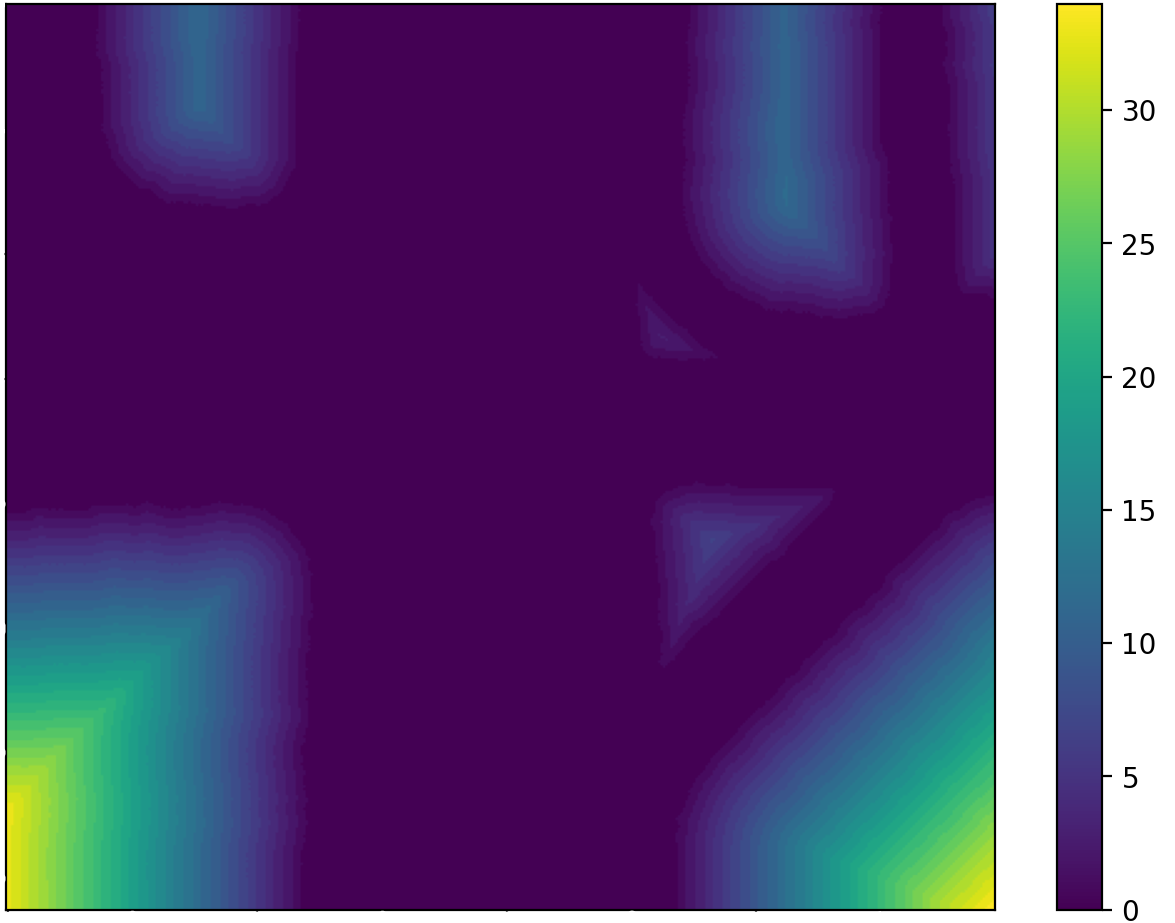}
		\caption{}\label{fig:sdf4}
	\end{subfigure}
	\caption{Static obstacle representation: (a) The original intersection. (b) Static obstacles are highlighted in white. (c) Signed-distance field. (d) Application of the Hinge Loss function.}
\end{figure*}

To incorporate the resultant expected uncertainty from AVs' trajectory following, we construct a regression model of controller uncertainty. The model takes as input the vehicle state and control, namely $[\zeta, v^c, a, \Dot{\zeta}]$, and returns the expected error defined as a radius with a predefined confidence interval. To this end, we first generate a set of reference trajectories (with uniform goal points as visualized in Fig.~\ref{fig:reg_model_paths}) using the cost function defined earlier  which incorporates all the potential trajectories at an intersection. Next, we execute each trajectory using the controller 
and collect the data. 
The data is then normalized by averaging over each subset of input trajectories and the expected error is calculated using the Quantile function. 
The linear regression model is then trained using the normalized input and the expected error. As empirically observed, the model attained 98.96\% accuracy.

Given two sets of vehicle trajectories $S_x^1$ and $S_x^2$, we let the dynamic distance $D_d$ between the two be the sum of the geometric distance (positive value implies overlapping) and the controllers' expected uncertainty. Mathematically,

{\footnotesize\begin{align*}
D_d = \max\Big\{ \Upsilon (o_1, o_2) +\tt{Reg}(s_1)+\tt{Reg}(s_2) ~ \forall~ s_1\in S_x^1, s_2\in S_x^2, ~ 0 \Big\} \,,  
\end{align*}}where $o_1,o_2$ are the circles that represent vehicle geometry per state, $\Upsilon (o_1, o_2) \triangleq \min\limits_{o_1\in s_1, o_2\in s_2} r_{o_1} + r_{o_2} - d(o_1,o_2)$ with $r_{o_1}, r_{o_2}$ denoting the radii of $o_1, o_2$ while $d(\cdot)$ capturing their Euclidean distance, and $\tt{Reg}(\cdot)$ encodes the output of regression models. Observe that $D_d = 0$ indicates no collision. The cost of dynamic obstacles is set to $c_d = e^{-\frac{1}{2}(D_d)^2}$.

\subsection{Trajectory Optimization}
We cast the trajectory optimization problem as a Shooting Method \cite{mayne1966second} and solve it using the Control-Limited Differential Dynamic Programming (DDP) algorithm \cite{tassa2014control}, which is an indirect method that admits quadratic convergence for any system with smooth dynamics \cite{jacobson1970differential} while bounding the control inputs. Hence, we are able to bound the acceleration based on the vehicle's specifications or the required action speed. The trajectory optimization algorithm, explained in Alg.~\ref{alg:traj}, consists of two loops. The first generates an initial trajectory with static obstacles, while the second one incorporates dynamic obstacles.
In Alg.~\ref{alg:traj}, $S_x$ is the set of trajectories' states, $S_u$ is the set of trajectories' controls, $\cA$ is the set of all actions defined by a starting and a goal point. The expression $c_d(S_u^i, S_u^j | (i,j) \nsubseteq \cN )$ represents the dynamic obstacle cost between $i$ and all actions $j$ that don't share a collision point from the set of collision points $\cN$. The boolean variable \textit{Converge} represents the stopping criterion, which can be conditioned on a minimum threshold of difference between the previous trajectories and the new ones, or alternatively fixed to a certain number of iterations.

     \begin{algorithm}[!h]
    \SetAlgoLined
    \SetKwInOut{Input}{input}\SetKwInOut{Output}{output}
    \Input{Set of all actions $\cA$}
    \Output{$S_x$, $S_u$}
    \For{$\cA^i \in \cA$}{
        $S_x^i, S_u^i = DDP(\cA^i)$ 
    }
    \While{Not Converge}{
    \For{$\cA^i \in \cA$}{
        $S_x^i, S_u^i = DDP(\cA^i, S_u^i, c_d(S_u^i, S_u^j | (i,j) \nsubseteq \cN ))$ \\
    }}
    \textbf{return} $S_x, S_u$
    \caption{Trajectory Optimization}
    \label{alg:traj}
    \end{algorithm}

  \begin{table*}[t!]
  \footnotesize
  \begin{minipage}{0.36\textwidth}
    \vspace{-3mm}
	\resizebox{\textwidth}{!}{
    {\renewcommand{\arraystretch}{1.52}
    \begin{tabular}{|p{3.5cm}|l|l|l|l|l|l|}
    \hline
    \multirow{2}{*}{}          & \multicolumn{6}{c|}{Risk Bound ($\Delta$)}       \\ \cline{2-7} 
                               & 0.01\% & 0.1\% & 1.0\% & 5.0\% & 10.0\% & 15.0\% \\ \hline
    FCFS (2 actions)           & 82     & 82    & 82    & 83    & 83     & 83     \\ \hline
    FCFS (3 actions)           & 83     & 84    & 94    & 96    & 97     & 103    \\ \hline
    MCC-SSP (2 actions; $h=1$) & 110    & 112   & 156   & 161   & 163    & 162    \\ \hline
    MCC-SSP (2 actions; $h=2$) & 106    & 113   & \textbf{157}   & \textbf{164}   & 160    & \textbf{166}    \\ \hline
    MCC-SSP (2 actions; $h=3$) & 95     & 112   & 146   & 155   & 157    & 158    \\ \hline
    MCC-SSP (3 actions; $h=1$) &\textbf{118}& \textbf{123}   & 156   & 157   & \textbf{165}    & 165    \\ \hline
    MCC-SSP (3 actions; $h=2$) & 105    & 120   & 151   & 157   & 157    & 160    \\ \hline
    \end{tabular}
    }
    }
    \caption{The throughput (vehicles/minute) of the intersection.}
    
    \label{tab:t1}
    \end{minipage}\hspace{5pt}
\begin{minipage}{0.3\textwidth}
    \centering
       \includegraphics[width=\linewidth]{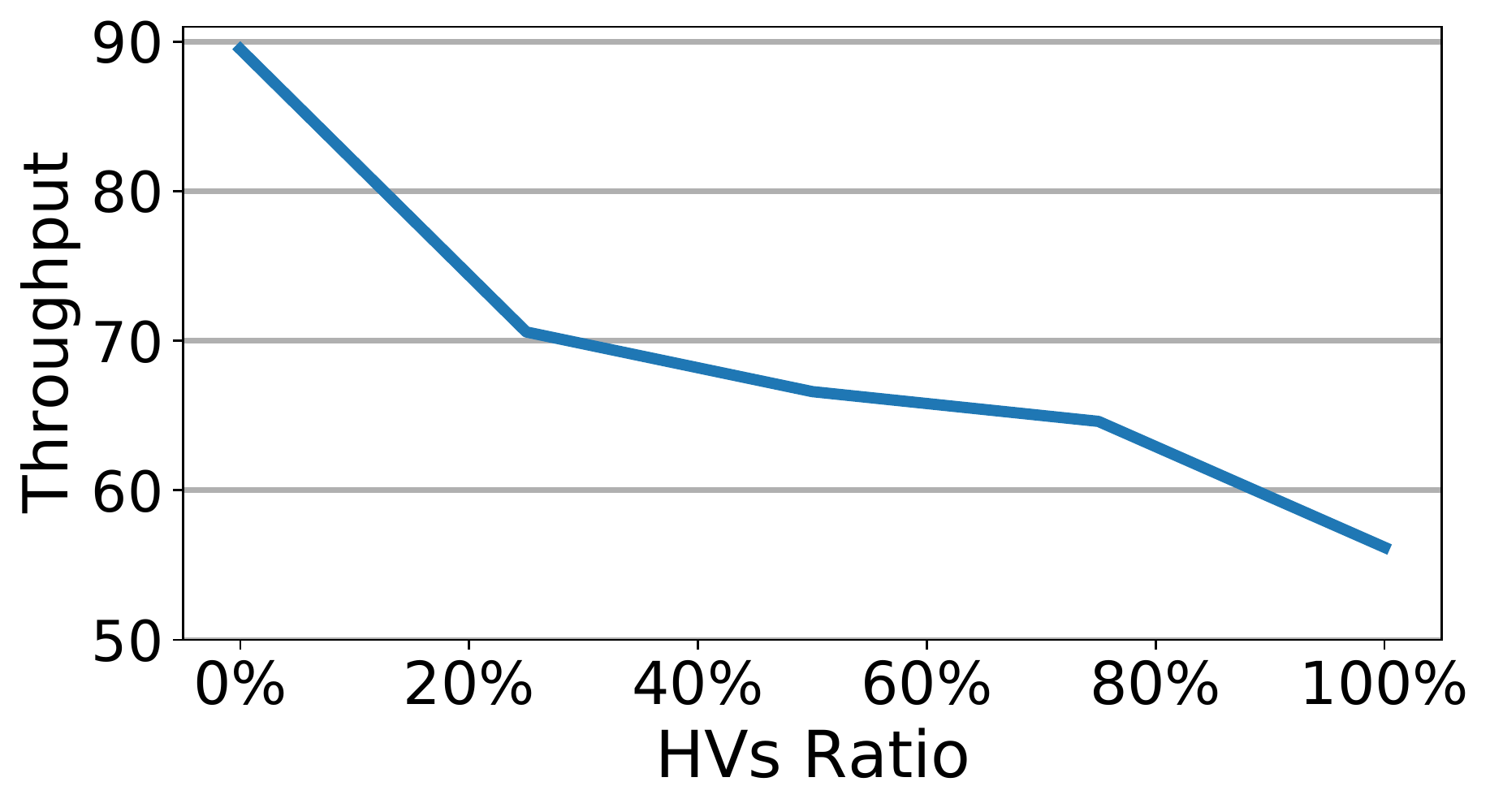}
       \vspace{-10pt}
	\captionof{figure}{The impact of HVs on the throughput (2 actions, $h$=3, $\Delta$=0.01\%).}
	\label{fig:HV-throughput}
    \vspace{-10pt}
	\end{minipage} \hspace{5pt}
	  \begin{minipage}{0.3\textwidth}
\centering
\resizebox{\textwidth}{!}{
\begin{tabular}{|l|l|l|l|}
\hline  
\multicolumn{2}{|c|}{($\Delta=5\%$)}               & \multicolumn{2}{c|}{Planning time (sec)} \\ \hline
Actions                    & Horizon & Preprocessing       & Solving       \\ \hline
\multirow{6}{*}{2 actions} & $h=1$   & 0.00361            & 0.0143           \\ \cline{2-4} 
                           & $h=2$   & 0.01379            & 0.0610           \\ \cline{2-4} 
                           & $h=3$   & 0.02700            & 0.2344           \\ \cline{2-4} 
                           & $h=4$   & 0.06879            & 0.6550           \\ \cline{2-4} 
                           & $h=5$   & 0.17493            & 2.1739            \\ \cline{2-4} 
                           & $h=6$   & 0.43763            & 6.6600            \\ \hline
\multirow{2}{*}{3 actions} & $h=1$   & 0.01426            & 0.0421           \\ \cline{2-4} 
                           & $h=2$   & 0.06661            & 1.0946           \\ \hline
\end{tabular}%
}
\caption{The planning time for 16 random AVs evenly distributed over a two-lane four-sided intersection.}
\vspace{-10pt}
\label{tab:planning_time}
\end{minipage} 
\end{table*}

\vspace{-5pt}
\section{Performance Evaluation} \label{sec:experiment}
To proceed with the evaluation, we first test the introduced ILP formulation's scalability on the multi-agent version of the well-known grid problem with independent agents (i.e., $\cX = \cN$) and shared risk constraint. Next, we employ the CARLA simulator \cite{dosovitskiy2017carla} as a test-bed to simulate the proposed risk-aware intelligent intersection system, verify its effectiveness and practicality as well as collect ground truth data on vehicle driving. In this section, we report the empirical findings on the planning time complexity of MCC-SSP (Table \ref{tab:planning_time}), the throughput of a fully-AV intersection (Table \ref{tab:t1}), the impact of HVs on the throughput (Fig. \ref{fig:HV-throughput}), and  experimentations on variants of the objective function (Fig.~\ref{fig:obj_exp}). Lastly, we present two case studies of the trajectory optimization workflow.

 \subsection{Scalability Analysis}
 As one demonstration, we apply the proposed MCC-SSP model to the multi-agent grid problem wherein robots can move in four directions inside a bounded discretized area. The movement, however, is uncertain with an 80\% success probability represented in the transition function, 5\% of the states are randomly defined as risky, and 10\% of the states are randomly set with a cost of $1$ while the rest have a cost of $2$. The grid size is set to (10000x10000) to assess the planner's performance at scale. The initial state is random for each agent. Fig. \ref{fig:multiagent-grid} plots the running time and the average objective value. As demonstrated by the figure, the formulation scales well with the horizon size and number of agents. In general, the running time is expected to increase with the number of agents and the length of planning horizon as more variables and constraints would be involved.

\begin{figure}[ht]
	\begin{subfigure}{.495\linewidth}
		\centering
		\includegraphics[width=\linewidth]{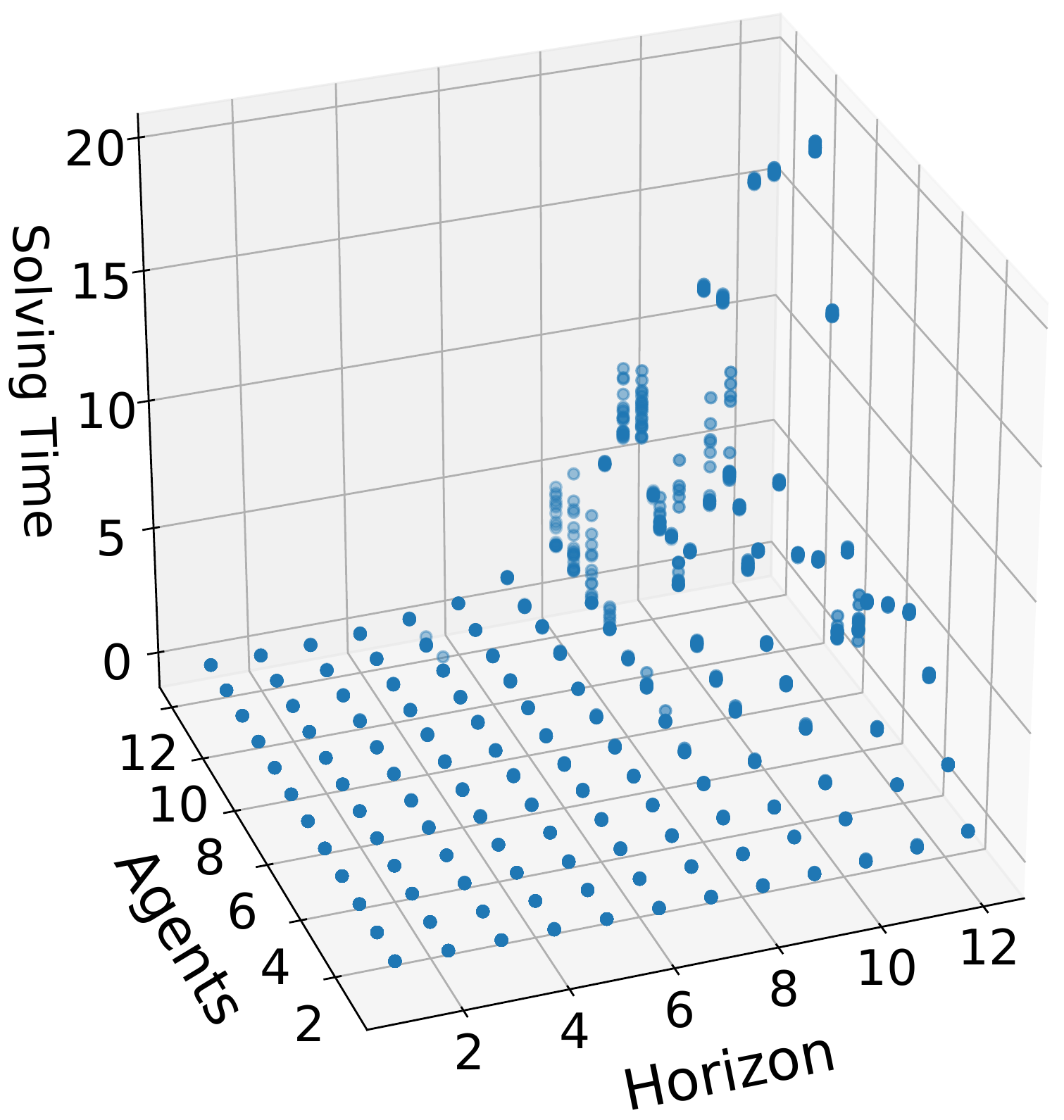}
		\caption{}
	\end{subfigure}
	\begin{subfigure}{.495\linewidth}
		\centering
		\includegraphics[width=\linewidth]{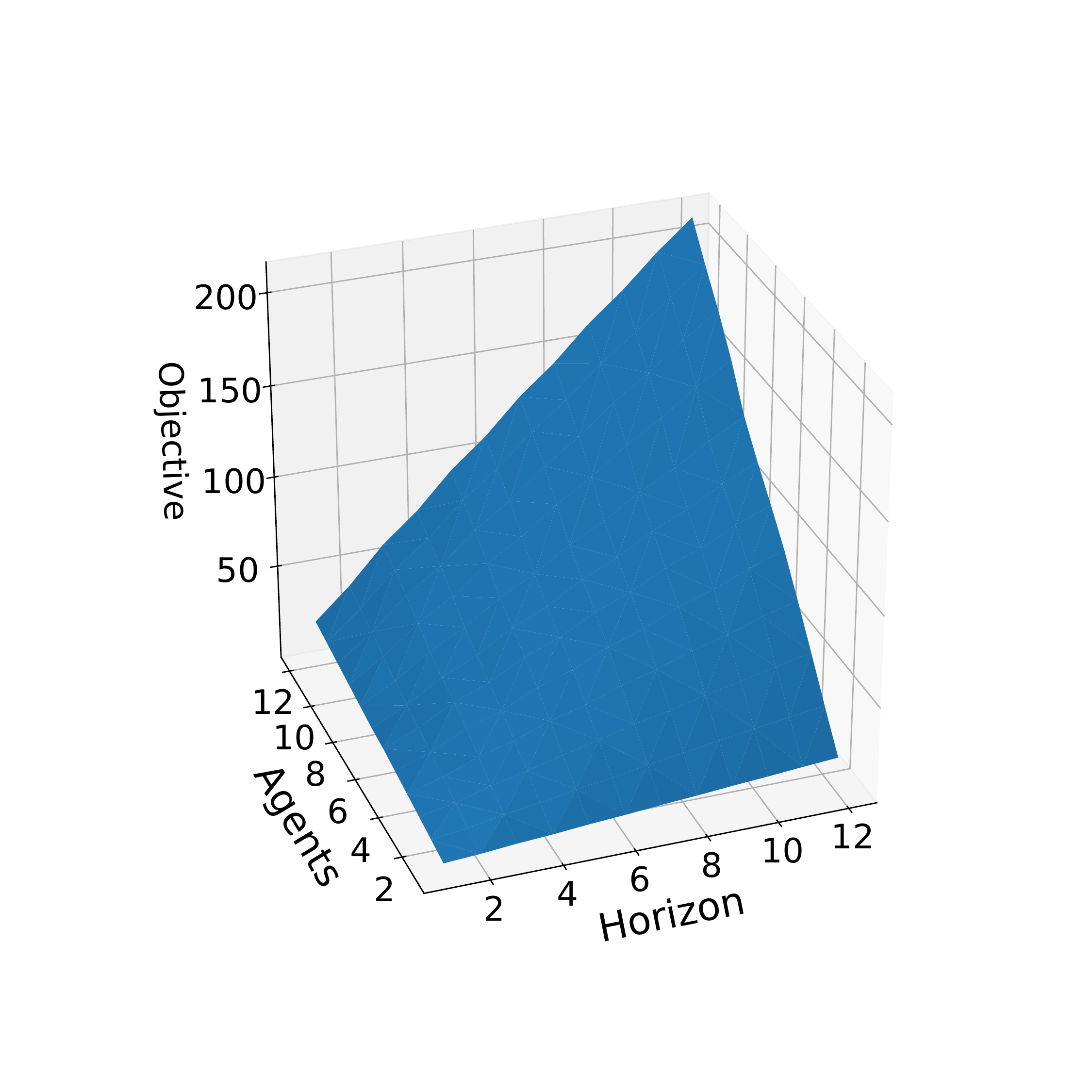}
		\caption{}
	\end{subfigure}
	\caption{Performance of the proposed MCC-SSP model on the multi-agent grid problem against the planning horizon and number of agents: (a) {\sf MCC-SSP-ILP} solving time, and (b) the objective value.}
   \label{fig:multiagent-grid}
\end{figure}

\subsection{Simulated Risk-aware Intelligent Intersection System}

\subsubsection{Intersection Throughput and Planning time}\hfill\\
\noindent
\textit{\bfseries Setup}:
CARLA simulator was used to generate PFTs for both AVs and HVs. For AVs, we executed a PID controller over nominal trajectories multiple times, each with slightly perturbed coefficients (uniformly chosen $P\in[0.4,1.2], I=0$, and $D\in[0.2,0.8]$)\footnote{Samples with  tracking error higher than 1 meter were excluded.}. The motivation behind this setup is that AVs from different vendors may have different controller setups. There are other factors, such as vehicle drift, weather, and road conditions, that could affect performance in real life. 
Similarly, we generate PFTs for HVs, with a 50\% chance of taking either action (e.g., go left or straight) when the traffic signal for HVs is green for the corresponding side (rotating every minute).  As an HV is accessing the intersection, the probability gradually approaches 100\% for the respective action. For comparison, we employ  the FCFS augmented with our risk detection approach as a benchmark planner. The simulations were repeated over 300-fold to reduce the uncertainties in the results.
The simulation setup consists of a two-lane four-sided (eight AVs) intersection (depicted in Figure \ref{fig:colide-points}), where the horizon duration is one second and a receding horizon is used for continuous planning. The horizon duration is the time between each planning horizon.
The trajectories are defined based on a PFT with 6Hz time-step.

\noindent
\textit{\bfseries Results}:
The first set of simulations, summarized in Table~\ref{tab:t1}, contrasts the performance (in terms of throughput) of FCFS and MCC-SSP under different risk thresholds and number of actions per agent. While the risk budget in MCC-SSP bounds the expected risk of the policy, the FCFS planner parses it as a bound for each action taken per agent and thus the expected risk in FCFS may exceed the bound for an MCC-SSP's single horizon. 
    As evident from Table~\ref{tab:t1}, MCC-SSP outperforms FCFS for any risk bound, and increasing the risk bound ($\Delta$) improves the performance since the system is taking more riskier actions. On the other hand, increasing the horizon doesn't necessarily improve the throughput. Even though a longer planning horizon provides a more optimal solution, the same risk bound gets distributed over the planning horizon. Thus, effectively, the single horizon case, for instance, has a higher risk threshold within, say, two receding horizons when compared to $h=2$. 
The observed planning time and scalability of the ILP formulation are reported in Table~\ref{tab:planning_time} as a function of the planning horizon and number of actions per agent. We presume that a planning time less than or equal to a single horizon duration (the horizon duration is one second) is reasonable for an intersection system; thus, for two actions we are able to use a horizon of four, and with three actions we are able to use a horizon of two.
We also present the negative impact of HVs on the throughput in Fig.~\ref{fig:HV-throughput}. As anticipated, with more human-driven vehicles, the throughput decreases.

\subsubsection{Intersection Objective Function}\hfill\\
\noindent
\textit{\bfseries Setup}:
To investigate the effectiveness of the objective function put forth in Sec.\ref{sec:objjj}, we performed a test case (portrayed in Fig.~\ref{fig:obj_exp1}) where an infinite number of AVs are arriving from the north and south. All AVs are traveling forward, similar to a highway scenario. On the other hand, the ego vehicle (green circle) is attempting to turn left, starting from the west and heading north.
We test two variants of the objective function, the first without the waiting time parameter ($\lambda_2=0$) and the second with the waiting time parameter ($\lambda_2=4$).

\vspace{4pt}
\noindent
\textit{\bfseries Results}:
In the case of the first objective, we observe that the ego vehicle never enters the intersection and waits indefinitely (depicted in Fig.~\ref{fig:obj_exp2}), which is undesirable. On the other hand, with the second objective, the ego vehicle enters the intersection (depicted in Fig.~\ref{fig:obj_exp3}) after the 10th horizon ($4\sqrt{10} > 15$ where 15 is the number of AVs waiting at the intersection).

\begin{figure}[ht]
	\begin{subfigure}{.327\linewidth}
		\centering
		\includegraphics[trim={15cm 25cm 20cm 25cm}, clip, width=\linewidth]{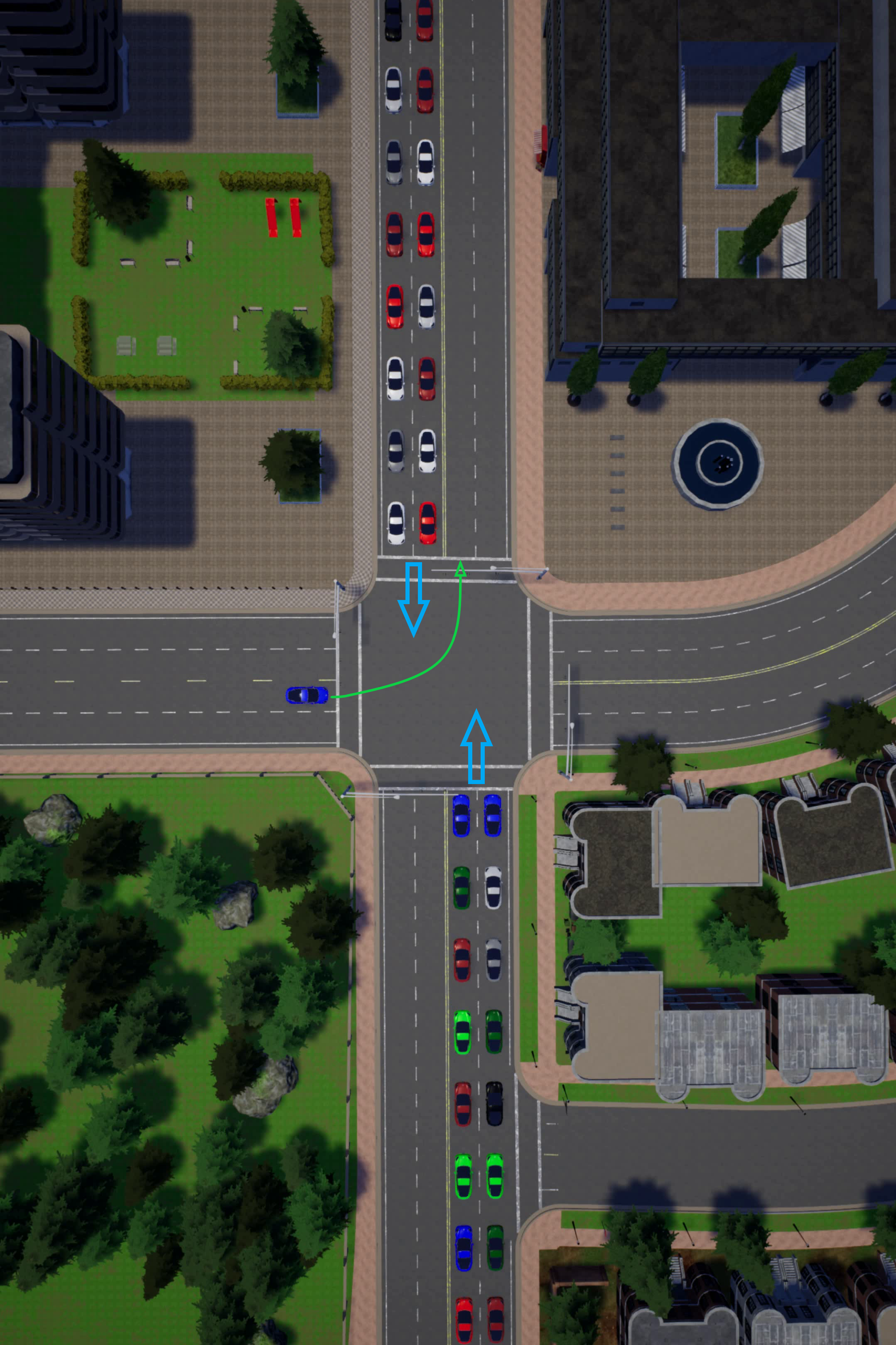}
		\caption{}
		   \label{fig:obj_exp1}
	\end{subfigure}
	\begin{subfigure}{.327\linewidth}
		\centering
		\includegraphics[trim={15cm 25cm 20cm 25cm},clip, width=\linewidth]{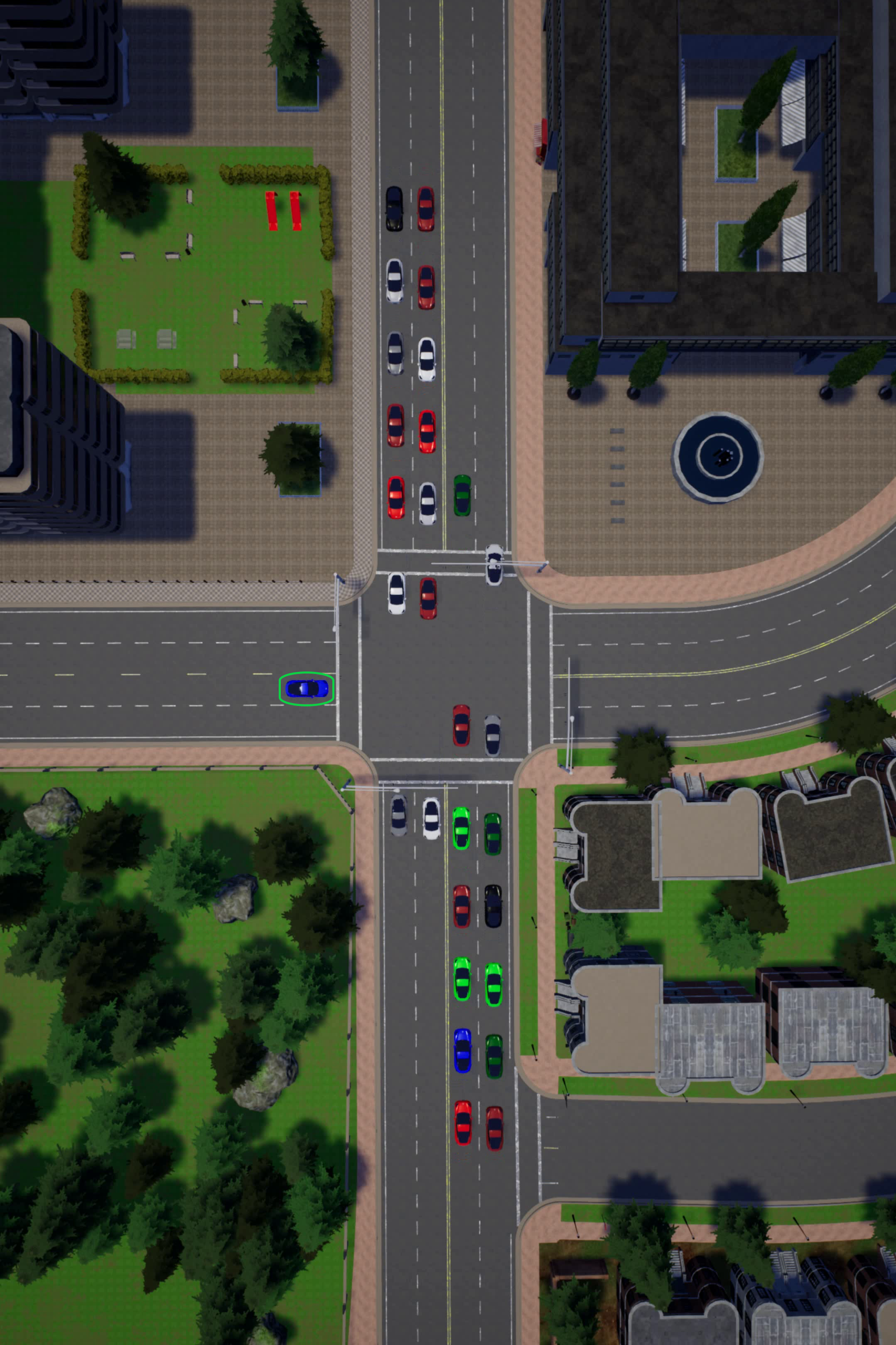}
		\caption{}
		   \label{fig:obj_exp2}
	\end{subfigure}
		\begin{subfigure}{.327\linewidth}
		\centering
		\includegraphics[trim={15cm 25cm 20cm 25cm},clip, width=\linewidth]{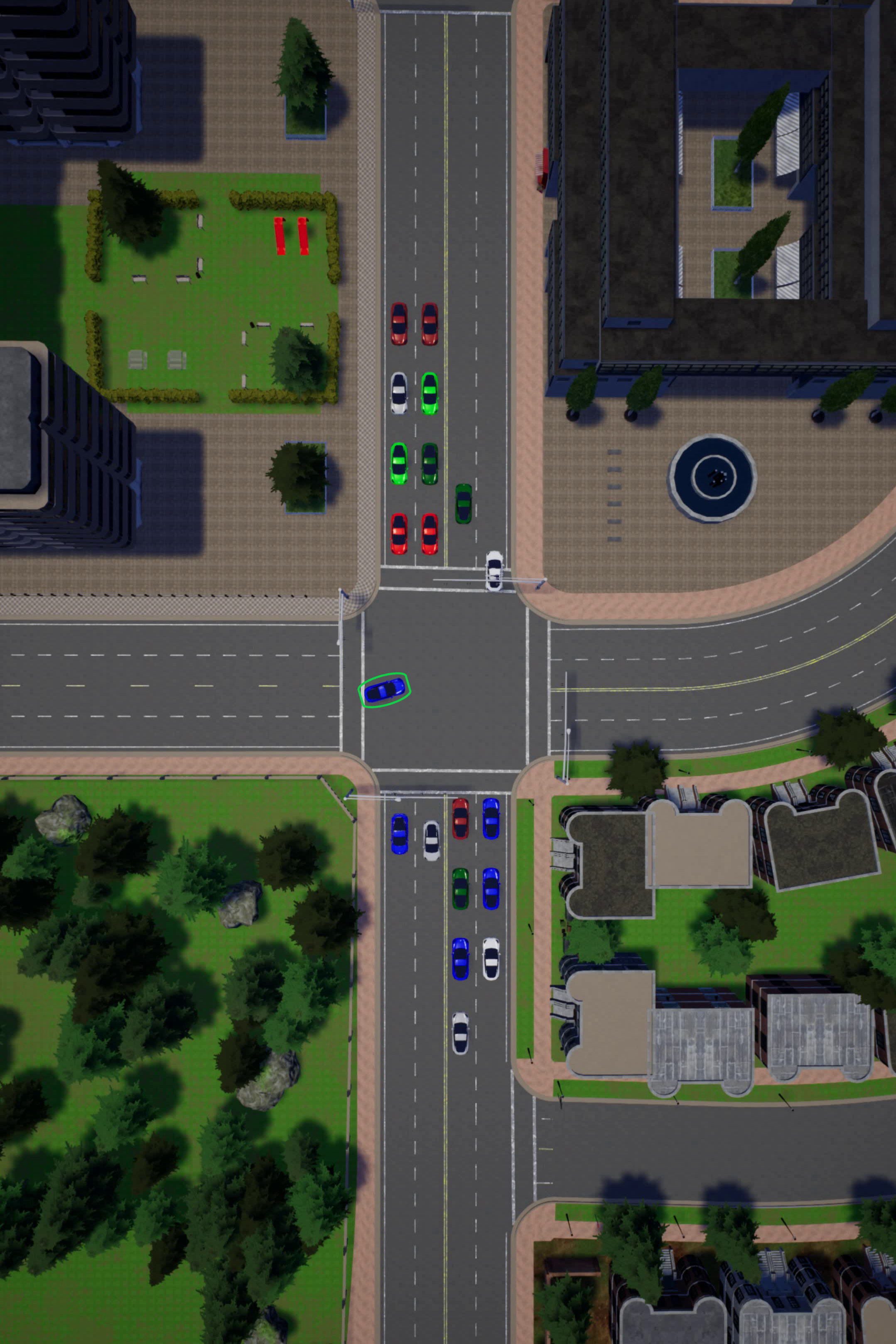}
		\caption{}
		   \label{fig:obj_exp3}
	\end{subfigure}
	\caption{Experimentation with the objective function: (a) The initial state and intention of the ego vehicle. (b) Running the planner with an objective function without waiting time. (c) Enacting the waiting time parameter in the objective function.}
	\label{fig:obj_exp}
\end{figure}

\vspace{-5pt}
\subsection{Trajectory Optimization}
To demonstrate the adopted trajectory optimization workflow, we ran our model on two intersections, one in Russia (Fig.~\ref{fig:intersection1}) with an asymmetrical number of the lanes and another intersection in the UAE (Fig.~\ref{fig:intersection2}) with multiple static obstacles in the center of the intersection. In the simulations, we utilized \textit{Crocoddyl} \cite{mastalli2020crocoddyl}, which is an open-source trajectory optimization software that implements the DDP algorithm.

The initial trajectories that were generated in the first loop of Alg.~\ref{alg:traj} were very close to each over and overlapped in some cases. However, after running the trajectory optimization with the dynamic obstacle cost (second loop of Alg.~\ref{alg:traj}), the trajectories diverged and no overlapping was observed.

  \begin{figure}[!h]
  	\begin{subfigure}{.515\columnwidth}
    \centering 
    \includegraphics[trim={1.5cm 0cm 0cm 0cm},clip,width=\textwidth]{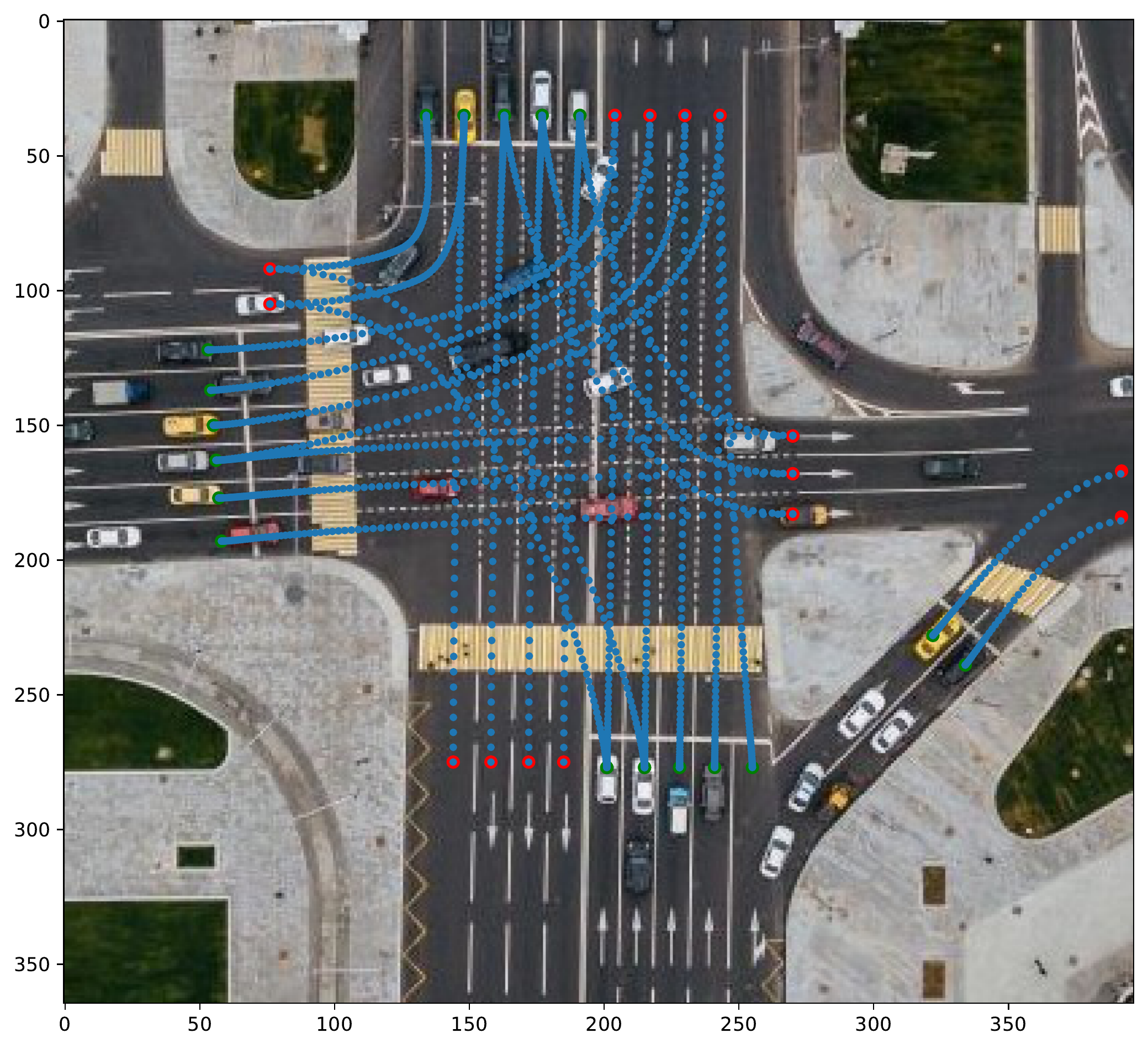}
 	\caption{}
 	\label{fig:intersection1}
 	\end{subfigure}
 	  \begin{subfigure}{.475\columnwidth}
    \centering 
    \includegraphics[width=\textwidth]{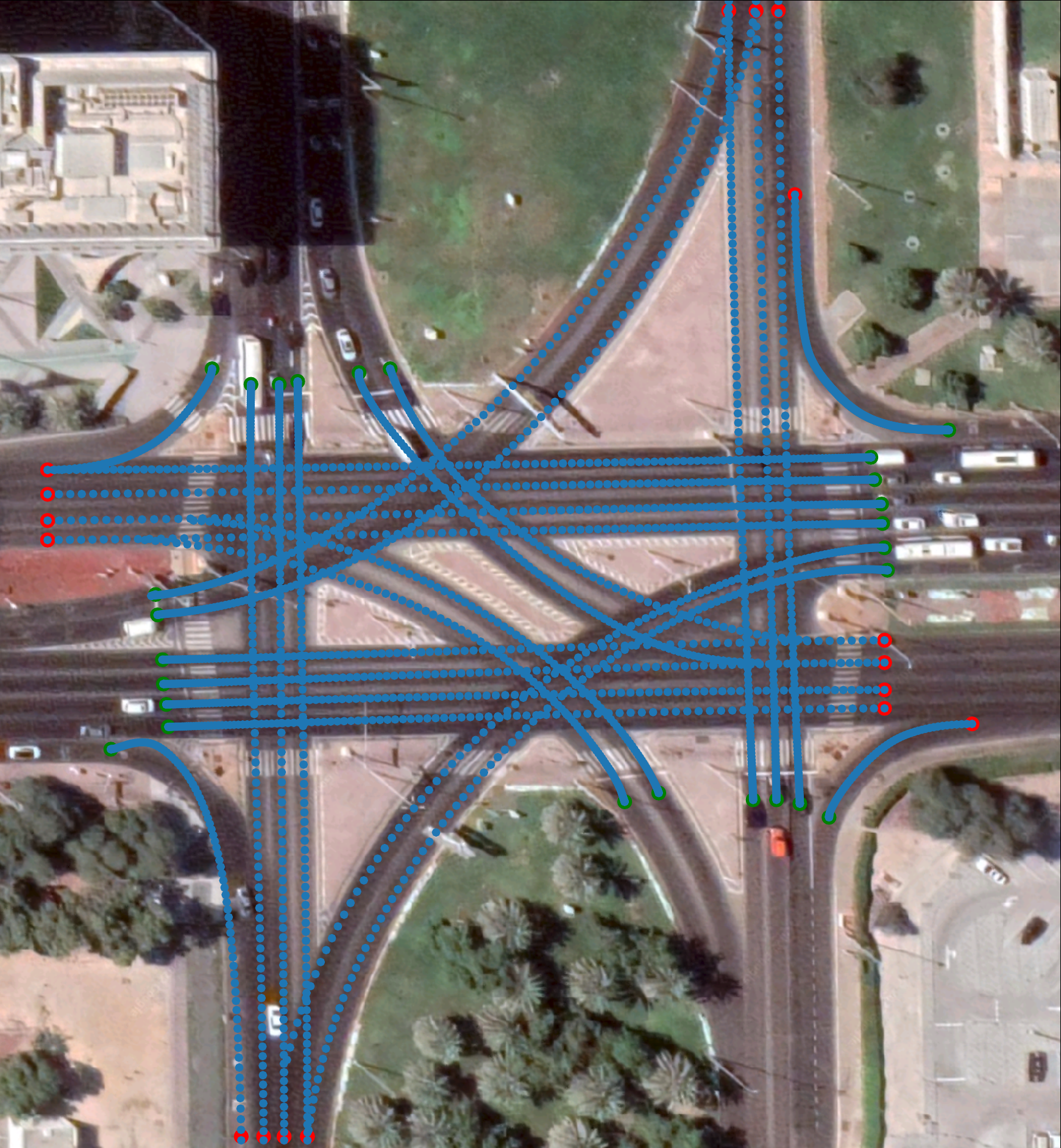}
 	\caption{}
 	\label{fig:intersection2}
 	\end{subfigure}
 	\caption{Output trajectories of the employed method for different intersections: (a) Tverskaya Zastava Square in Moscow, Russia. (b) Intersection near Khalifa University, Abu Dhabi, UAE.}
 	\end{figure}

\section{Conclusion} \label{sec:conclusion}
This work proposes a risk-aware Intelligent Intersection system modeled as a novel class of MCC-SSP, wherein agents interact at certain localized zones (collision points).
The system admits an adjustable risk tolerance parameter that allows to enforce the desired guarantee level on the probability of collisions despite the presence of perception and planning uncertainties. We introduce an exact integer linear programming formulation of the problem, featuring polynomial number variables and constraints when the number of agents per localized zone is small. The system is demonstrated in a realistic driving simulator that involves both AVs and HVs. As validated through simulations, the proposed system provides optimal plans that translate to higher throughput than the existing approaches. In future work, we target to design an approximation algorithm for the problem that runs in polynomial time and provides certifiable worst-case performance guarantees (i.e., approximation ratio).

\bibliographystyle{IEEEtran}
\bibliography{paper}

\appendix

\subsection{Proof of Lemma \ref{lem:er}}\label{lemproof}

The execution risk at $i$  can be written as

\footnotesize{
\begin{align}
&\textsc{Er}^j(s^i_k) =  1- \Pr\Big(\bigwedge_{t=k}^{h} \bigwedge_{v, v' \in \cX^i}\neg R^{j}(S^v_t, S^{v'}_t)  \mid S^i_k = s^i_k\Big). \notag \\[-8pt]
&= 1 - \bigg(\Pr\Big(\bigwedge_{\mathclap{t=k+1}}^{h} \neg R^{j}(S^i_t)  \mid S^i_k = s^i_k, \neg R^j(S^i_k) \Big)\Pr\Big( \neg R^{j}(S^i_k)  \mid S^i_k = s^i_k\Big) \notag  \bigg)\\[-3pt]
&= 1 - \Pr\Big(\bigwedge_{\mathclap{t=k+1}}^{h}\neg R^j(S^i_t) \mid S^i_k = s^i_k, \neg R^{j}(S^i_k) \Big) \prod_{\mathclap{v,v' \in \cX^i}}(1-r^{j}(s^v_k, s^{v'}_k)), \notag\\[-4pt]
&= 1 - \Pr\Big(\bigwedge_{t=k+1}^{h}\neg R^j(S^i_t) \mid S^i_k = s^i_k, \neg R^{j}(S^i_k) \Big) \overline r^{j}(s^i_k), \label{eq:er1}
\end{align}}
\normalsize{}

where $\neg R^j(S^i_k) := \bigwedge_{v, v' \in \cX^i}\neg R^{j}(S^v_k, S^{v'}_k)$ denote the event of being safe at time $k$ in interaction point $i$, and $\overline r^j(s^i):= \prod_{v,v' \in \cX^i}(1-r^{j}(s^v_k, s^{v'}_k))$ denote its corresponding probability.
The probability term in the last equation can be expanded, conditioned over subsequent states at time $k+1$,

\vspace{-5pt}
\begin{align}
&\sum_{s^i_{k+1}\in \cS^i} \Pr\Big(\bigwedge_{\mathclap{t=k+1}}^{h} \neg R^{j}(S^i_t)  \mid S^i_k = s^i_k, \neg R^{j}(S^i_k), S^i_{k+1} = s^i_{k+1} \Big) \notag\\[-6pt]
&\qquad\qquad\qquad\qquad\cdot \Pr\Big(S^i_{k+1} = s^i_{k+1} \mid S^i_k = s^i_k, \neg R^{j}(S^i_k) \Big) \notag \\
&= \sum_{s^i_{k+1}} \Pr\Big(\bigwedge_{t=k+1}^{h}\neg R^{j}(S^i_k)  \mid S^i_{k+1} = s^i_{k+1}\Big) \notag \\[-7pt]
&\qquad\qquad\qquad\qquad\qquad \cdot \Pr(S^i_{k+1} = s^i_{k+1} \mid S^i_k = s^i_k) \label{eq:ind'}  \\
& = \sum_{s^i_{k+1}} (1 - \textsc{Er}^j(s^i_{k+1}) )  T^i(s^i_k, \pi(s^i_k), s^i_{k+1}), \label{eq:er2}
\end{align}
where Eqn.~\raf{eq:ind'} follows from the independence between $\big(\bigwedge_{t=k+1}^{h} \neg R^j(S^i_t) \mid S^i_{k+1}= s^i_{k+1}\big)$ and $(S^i_k = s^i_k\wedge \neg R^j(S^i_k))$, and between $(S_{k+1} = s_{k+1} \mid S_k = s_k)$ and $\neg R^j(S^i_k)$.
Combining Eqns.~\raf{eq:er1},\raf{eq:er2} obtains $\textsc{Er}^j(s^i_k)$ as \vspace{-5pt}

\begin{align}
&=  1 - \overline r^{j}(s^i_k) \sum_{\mathclap{s^i_{k+1}}} (1-\textsc{Er}^j(s^i_{k+1}) )   T^i(s^i_k, \pi(s^i_k), s^i_{k+1}), \notag \\
&= 1 +\overline r^{j}(s^i_k)  \Big[\sum_{s^i_{k+1}} \textsc{Er}^j(s^i_{k+1}) T^i(s^i_k, \pi(s^i_k), s^i_{k+1})  \notag \\
& \qquad - \sum_{s^i_{k+1}}   T^i(s^i_k, \pi(s^i_k), s^i_{k+1}) \Big] \notag \\
&= 1 +\overline r^j(s^i_k) \Big[ \sum_{s^i_{k+1}} \textsc{Er}^j(s^i_{k+1})T^i(s^i_k, \pi(s^i_k), s^i_{k+1}) - 1\Big] \notag \\
&=1 - \overline r^j(s^i_k) +\overline r^{j}(s^i_k)  \sum_{s^i_{k+1}} \textsc{Er}^j(s^i_{k+1}) T^i(s^i_k, \pi(s^i_k), s^i_{k+1}) \notag \\
&= 1-\overline r^j(s^i_k) +  \sum_{s^i_{k+1}} \textsc{Er}^j(s^i_{k+1}) \widetilde T^{i,j}(s_k, \pi(s_k), s_{k+1}), \label{eq:er3}
\end{align}
where  $\widetilde T^{i,j}(s^i_k, \pi(s^i_k), s^i_{k+1})=  T^i(s^i_k, \pi(s^i_k),
s^i_{k+1}) \overline r^{j}(s^i_k)$.
For a stochastic policy $\pi$, Eqn.~\raf{eq:er3} can be easily written as claimed in the lemma.

\subsection{Proof of Theorem~\ref{lem1} }\label{thmproof}
\begin{proof}
We can rewrite the execution risk from Lemma \ref{lem:er} as $\textsc{Er}'^j(s^i_k) =$
\begin{equation}
 \left\{\begin{array}{l}
\widetilde r^j(s_0) + \sum_{s^i_{k+1}\in \cS^i}\sum_{a^i \in \cA^i} ( \textsc{Er}'^j(s^i_{k+1}) + \widetilde r^j(s^i_{k+1}) ) \notag \\
 \hspace{10mm} \cdot \pi(s^i_k, a^i) \widetilde T^{i,j}(s^i_k, a^i, s^i_{k+1}) \hspace{7mm} \text{if } k = 0, \\
\sum_{s^i_{k+1}\in \cS^i}\sum_{a^i \in \cA^i} ( \textsc{Er}'^j(s^i_{k+1}) + \widetilde r^j(s^i_{k+1}) ) \notag\\
 \hspace{10mm} \cdot \pi(s^i_k, a^i) \widetilde T^{i,j}(s^i_k, a^i, s^i_{k+1})\hspace{7mm} \text{if } k = 1,...,h-1,\\
\hspace{23mm} 0 \hspace{29mm} \text{if } k=h.
\end{array}\right. \label{eq:er0}
\end{equation}
where $\textsc{Er}^j(s^i_k) = \textsc{Er}'^j(s^i_k)  + \widetilde r^j(s^i_k)$, and $\textsc{Er}^j(s_0) = \textsc{Er}'^j(s_0)$.
Based on the  flow equations \raf{conf1},\raf{conf2}, define a policy 
\begin{equation}
\pi(s^i_k,a^i) := \left\{ \begin{array}{l l}
 x^{i,j}_{s,0,a} \hspace{30pt}\vspace{10pt},& \text{if } s^i_k = s_0 \\
 \frac{x_{s,k,a}^{i,j}}{\sum_{a' \in \cA^i}  x_{s,k,a'}^{i,j}}, & \text{otherwise.}
\end{array}\right.
\end{equation}
Note that the policy is a valid probability distribution. Thus, we rewrite Eq.~\raf{conf1},\raf{conf2} by
\begin{align}
    & x^{i,j}_{s,k,a} = \pi(s^i_k,a^i) \sum_{s^i_{k-1}\in\cS^i}\sum_{a'\in \cA^i} x^{i,j}_{s,k-1,a'} \widetilde T^{i,j}(s^i_{k-1},a',s^i_k) \notag \\
    & x^{i,j}_{s,0,a} = \pi(s_0,a^i). \label{eq:pix}
\end{align}
Next, we proof by induction the following statement
\begin{align}
&\textsc{Er}^j(s_0) = \sum_{i \in \cN} \widetilde r^j(s^i_0) \notag \\ & \hspace{5mm} + \sum_{k = 1}^{h'} \sum_{i\in\cN~~} \sum_{\mathclap{\substack{s^i_{k-1}\in \cS^i \\a^i \in \cA^i, s^i_k\in \cS^i }}} \widetilde r^j(s_{k}) x^{i,j}_{s,k-1,a} \widetilde T^{i,j}(s^i_{k-1}, a^i, s^i_{k}) \notag \\
&\hspace{5mm} +  \sum_{i\in\cN~~} \sum_{\mathclap{\substack{s^i_{{h'}-1}\in \cS^i \\a^i \in \cA^i, s^i_{h'}\in \cS^i }}}  \textsc{Er}'^j(s^i_{h'})  x^{i,j}_{s,h'-1,a} \widetilde T^{i,j}(s^i_{h'-1}, a^i, s^i_{h'}), \label{eq:ind}
\end{align}
Note that when $h' = h$, by Eqn.~\raf{eq:er0}, the last term of the above equation is zero, which is equivalent to the lemma's claim.
We consider the initial case with $h'=1$. From Eqn.~\raf{eq:er0} and $\textsc{ER}^j(s_0) = \textsc{ER}^{'j}(s_0)$, we obtain
\begin{align*}
    &\textsc{Er}^j(s_0) =\sum_{i \in \cN} \widetilde r^j(s^i_0) \notag \\ & \hspace{6mm} + \sum_{i\in\cN~~} \sum_{\mathclap{\substack{s^i_{1}\in \cS^i \\ a^i \in \cA^i}}} (\widetilde r^j(s^i_{1}) + \textsc{Er}'^j(s^i_{1})  ) \pi(s^i_0, a^i) \widetilde T^{i,j}(s_0^i, a^i, s^i_{1}) \\ 
     &=\sum_{i \in \cN} \widetilde r^j(s^i_0) + \sum_{i\in\cN~~}\sum_{\mathclap{\substack{s^i_{1}\in \cS^i \\ a^i \in \cA^i}}}\widetilde r^j(s^i_{1}) x^{i,j}_{s,0,a} \widetilde T^{i,j}(s_0^i, a^i, s^i_{1}) \notag \\ & \hspace{20mm} + \sum_{i\in\cN~~}\sum_{\mathclap{\substack{s^i_{1}\in \cS^i \\ a^i \in \cA^i}}} \textsc{Er}'^j(s^i_{1}) x^{i,j}_{s,0,a} \widetilde T^{i,j}(s_0^i, a^i, s^i_{1}),
\end{align*}
where $s_{h'-1} = s_0$ is a known state.
For the inductive step, we assume  Eqn.~\raf{eq:ind} holds up to $h'=t$, we proof the statement for $h'=t+1$. Expanding $\textsc{Er}^{'j}(s_t)$ using Eqn.~\raf{eq:er0} obtains
\begin{align}
 &\textsc{Er}^j(s_0) =\sum_{i \in \cN} \widetilde r^j(s^i_0) \notag \\ & \hspace{5mm} + \sum_{k=1}^t \sum_{i\in\cN~~} \sum_{\mathclap{\substack{s^i_{k-1}\in \cS^i \\a^i \in \cA^i, s^i_k\in \cS^i }}} \Big(\widetilde r^j(s^i_{k}) x^{i,j}_{s,k-1,a} \widetilde T^{i,j}(s^i_{k-1}, a^i, s^i_{k}) \Big) \notag \\
 &\hspace{5mm}  + \sum_{i\in\cN~~} \sum_{\mathclap{\substack{s^i_{t-1}\in \cS^i \\a^i \in \cA^i, s^i_t\in \cS^i }}}
 \big( \textsc{Er}'^j(s^i_{t}) x^{i,j}_{s,t-1,a} \widetilde T^{i,j}(s^i_{t-1},a^i,s^i_t) \big) \notag
 \end{align}
    \begin{align}
     &= \sum_{i \in \cN} \widetilde r^j(s^i_0) \notag \\ & \hspace{5mm}+ \sum_{k=1}^t \sum_{i\in\cN~~} \sum_{\mathclap{\substack{s^i_{k-1}\in \cS^i \\a^i \in \cA^i, s^i_k\in \cS^i }}} \Big(\widetilde r^j(s^i_{k}) x^{i,j}_{s,k-1,a} \widetilde T^{i,j}(s^i_{k-1}, a^i, s^i_{k}) \Big) \notag\\
    & \hspace{5mm} + \sum_{i\in\cN~~} \sum_{\mathclap{\substack{s^i_{t}\in \cS^i \\{a'}^{i} \in \cA^i, s^i_{t+1}\in \cS^i }}}
 \Big( \big( \widetilde r^j(s^i_{t+1}) + \textsc{Er}'^j(s^i_{t+1})  \big)\pi(s^i_t,{a'}^{i}) \notag \\ & \hspace{5mm} \cdot  \widetilde T^{i,j}(s^i_t,a{i'},s^i_{t+1})  \sum_{s^i_{t-1}\in \cS^i}\sum_{a^i \in \cA^i} \big(x^{i,j}_{s,t-1,a}   \widetilde T^{i,j}(s^i_{t-1}, a^i, s^i_{t} \big)  \Big) \notag
    \end{align}
    \begin{align}
     &=\sum_{i \in \cN} \widetilde r^j(s^i_0) \notag\\& +  \sum_{k=1}^t \sum_{i\in\cN~~} \sum_{\mathclap{\substack{s^i_{k-1}\in \cS^i \\a^i \in \cA^i, s^i_k\in \cS^i }}} \Big(\widetilde r^j(s^i_{k}) x^{i,j}_{s,k-1,a} \widetilde T^{i,j}(s^i_{k-1}, a^i, s^i_{k}) \Big) \notag \\
     &  + \sum_{i\in\cN~~} \sum_{\mathclap{\substack{s^i_{t}\in \cS^i \\a^i \in \cA^i, s^i_{t-1}\in \cS^i }}}
 \Big( \big( \widetilde r^j(s^i_{t+1}) + \textsc{Er}'^j(s^i_{t+1})  \big) \notag\\&\hspace{40mm} \cdot x^{i,j}_{s,t,a} \widetilde T^{i,j}(s^i_t,a^i,s^i_{t+1}) \Big) \label{eq:simpx}  \end{align}
where Eqn.~\raf{eq:simpx}  follows by substituting $\pi(s^i_t,a')$, using Eqn.~\raf{eq:pix}, by
$$ \pi(s^i_t,{a'}^{i})  = \frac{ x^{i,j}_{s,t,a'}} {\sum_{s^i_{t-1}\in\cS^i}\sum_{a^i\in \cA^i} x^{i,j}_{s,t-1,a} \widetilde T^{i,j}(s^i_{t-1},a^i,s^i_t)}.$$
Rewriting Eqn.~\raf{eq:simpx}, we obtain
\begin{align}
&\textsc{Er}^j(s_0)=\sum_{i \in \cN} \widetilde r^j(s^i_0)) \notag\\&\hspace{2mm} + \sum_{k=1}^{t+1} \sum_{i\in\cN~~} \sum_{\mathclap{\substack{s^i_{k-1}\in \cS^i \\a^i \in \cA^i, s^i_k\in \cS^i }}} \widetilde r^j(s^i_{k}) x^{i,j}_{s,k-1,a} \widetilde T^{i,j}(s^i_{k-1}, a^i, s^i_{k})\notag\\
& \hspace{2mm} + \sum_{i\in\cN~~} \sum_{\mathclap{\substack{s^i_{t}\in \cS^i \\a^i \in \cA^i, s^i_{t-1}\in \cS^i }}}
 \textsc{Er}'^j(s^i_{t+1})   x^{i,j}_{s,t,a} \widetilde T^{i,j}(s^i_t,a^i,s^i_{t+1}),
\end{align}
\end{proof}

\end{document}